\documentclass[fleqn]{article}
\usepackage{alex}
\usepackage{times}

\usepackage{amsthm}
\newtheorem{lemm}{Lemma}

\usepackage{natbib}
\usepackage{xspace}

\usepackage{etoolbox}
\newtoggle{show_main_figs}
\toggletrue{show_main_figs}
\newtoggle{show_appendix_figs}
\toggletrue{show_appendix_figs}
\newtoggle{show_figs}
\togglefalse{show_figs}
\newtoggle{include_updates_comparison}
\togglefalse{include_updates_comparison}
\newtoggle{old_figs}
\togglefalse{old_figs}

\newtoggle{fig_vertical}
\togglefalse{fig_vertical}

\usepackage{tikz,amsmath,amssymb, amsfonts, amsthm,bm}
\usepackage{enumitem}
\usepackage{float}
\usepackage[accepted]{icml2017}

\definecolor{colora}{rgb}{1,0,0}
\definecolor{colorab}{rgb}{1,1,0}
\definecolor{colorad}{rgb}{0,1,0}
\definecolor{colorbd}{rgb}{0,1,1}
\definecolor{colorbc}{rgb}{0,0,1}
\definecolor{colorcd}{rgb}{1,0,1}
\definecolor{colord}{rgb}{0.87,0.49,0}
\definecolor{colorb}{rgb}{0.5,0.5,0.5}
\definecolor{colorc}{rgb}{0.5,0.5,0.5}
\definecolor{boxcolor}{rgb}{0.9,0.9,0.9}

\theoremstyle{definition}

\newtheorem{mybox}{Box}

\newcommand{\vect}[1]{{\text{vec}\br{#1}}}

\DeclareMathOperator{\E}{\mathbb{E}}

\newcommand{\ave}[1]{\E\sq{#1}}

\newcommand{\pdiff}[2]{\frac{\partial^2}{\partial #1\partial #2}}

\renewcommand{\eqref}[1]{equation(\ref{#1})}

\newcommand{\m}[1]{{#1}}

\newcommand{\bmp}[1]{\begin{minipage}{#1}}
\newcommand{\bmpp}[2]{\begin{minipage}[#1]{#2}}
\newcommand{\emp}{\end{minipage}}

\tikzstyle{cont}=[circle,draw=blue!50,thick,minimum size=6mm,line width=2pt,>=stealth]  
\tikzstyle{ocont}=[ellipse,draw=blue!50,thick,minimum size=6mm,line width=2pt,>=stealth]  
\tikzstyle{blackcont}=[circle,draw=black!50,thick,minimum size=6mm,line width=2pt,>=stealth]  
\tikzstyle{oval}=[ellipse,draw=blue!50,thick,minimum size=6mm,line width=2pt,>=stealth]  
\tikzstyle{disc}=[rectangle,draw=blue!50,thick,minimum size=6mm]  
\tikzstyle{obs}=[fill=blue!20,thick]  

\tikzstyle{fillred}=[fill=red!20,thick]  
\tikzstyle{fillgreen}=[fill=green!20,thick]  
\tikzstyle{purered}=[fill=red]  
\tikzstyle{state}=[rectangle,fill=red!20]  
\tikzstyle{traffic}=[rectangle]  
\tikzstyle{sobs}=[fill=green!15,thick]  
\tikzstyle{fact}=[fill,minimum size=1.5mm,line width=2pt,>=stealth]
\tikzstyle{varfact}=[draw,minimum size=1.5mm,line width=2pt,>=stealth]
\tikzstyle{sep}=[rectangle,draw=magenta!50,thick,minimum size=6mm]  
\tikzstyle{det}=[fill=red!15,rectangle,draw=red!50,thick,minimum size=6mm]  

\tikzstyle{dethid}=[diamond,draw=red!50,thick,minimum size=6mm]  

\tikzstyle{lineball}=[fill,-*,draw=red!50,line width=1.5pt]
\tikzstyle{redball}=[mark=*,mark options={fill=red!50,draw=red},mark size=0.5pt]
\tikzstyle{greenball}=[mark=*,mark options={fill=green!50,draw=green},mark size=0.5pt]
\tikzstyle{hid}=[circle,draw,thick]  

\tikzstyle{dec}=[rectangle,draw=red!50,thick,minimum size=6mm]  
\tikzstyle{utility}=[diamond,draw=red!50,thick,minimum size=6mm]  
\tikzstyle{contdec}=[circle,draw=blue!50,thick,fill=blue!10,line width=2pt]  
\tikzstyle{decutility}=[diamond,draw=red!50,thick,minimum size=6mm]  

\tikzstyle{contobs}+=[cont]
\tikzstyle{contobs}+=[obs]
\tikzstyle{discobs}+=[disc]
\tikzstyle{discobs}+=[obs]

\tikzstyle{obsred}+=[obs]
\tikzstyle{obsred}+=[red]

\tikzstyle{background grid}=[draw, black!50,step=.1cm]
\tikzstyle{dgraph}=[->, line width=1.5pt]
\tikzstyle{ugraph}=[line width=1.5pt]

\newcommand{\y}{y}

\newcommand{\x}{x}

\newcommand{\chapref}[1]{chapter XX}

\newcommand{\br}[1]{\left( {#1} \right)}
\newcommand{\sq}[1]{\left[ {#1} \right]}

\newcommand{\cut}[1]{}

\newcommand{\trans}{^{\textsf{T}}}

\newcommand{\ocm}{\hspace{1cm}}

\newcommand{\xhess}[1]{{\cal{H}}_{#1}}
\newcommand{\gn}[1]{{\cal{G}}_{#1}}
\newcommand{\qn}[1]{{\cal{Q}}_{#1}}
\newcommand{\fn}[1]{{\cal{F}}_{#1}}

\newcommand{\gradient}{g}

\newcommand{\whess}{H}

\newcommand{\pre}{{h}}
\newcommand{\act}{{a}}
\renewcommand{\x}{{x}}
\renewcommand{\y}{{y}}
\newcommand{\nonl}{{f}}

\newcommand{\fullh}{H}
\newcommand{\W}{W}

\newcommand{\diag}[1]{\text{diag}\left({#1}\right)}
\newcommand{\Wlab}{\W^\lambda_{a,b}}
\newcommand{\Wlcd}{\W^\lambda_{c,d}}

\newcommand{\prea}{{pre-activation}\xspace}
\newcommand{\transfer}{transfer function\xspace}
\newcommand{\transfers}{transfer functions\xspace}
\newcommand{\params}{{{\theta}}}
\newcommand{\Real}{\mathbb{R}}

\newcommand{\pmodel}{p}

\newcommand{\ssind}[2]{{#1}_{{#2}}}
\newcommand{\J}[2]{{J^{#1}_{#2}}}
\newcommand{\Jth}{\J{\pre_L}{\params}}

\newcommand{\loggrad}[2]{\nabla_{#1} \log p_\params(#2)}
\newcommand{\loggradc}[1]{\loggrad{#1}{\y|\x}}

\newcommand{\Clk}{{C_\lambda^k}}

\newcommand{\Clpk}{{C_{\lambda+1}^k}}
\newcommand{\wtl}[1]{\widetilde{#1}_\lambda}

\DeclareMathOperator{\trace}{trace}

\let\oldbibitem\bibitem
\def\bibitem{\vfill\oldbibitem}

\icmltitlerunning{Practical Gauss-Newton Optimisation for Deep Learning}

\begin{document}
	
\twocolumn[
\icmltitle{Practical Gauss-Newton Optimisation for Deep Learning}

\begin{icmlauthorlist}
\icmlauthor{Aleksandar Botev}{ucl}
\icmlauthor{Hippolyt Ritter}{ucl}
\icmlauthor{David Barber}{ucl,ati}

\end{icmlauthorlist}

\icmlaffiliation{ucl}{University College London, London, United Kingdom}
\icmlaffiliation{ati}{Alan Turing Institute, London, United Kingdom}

\icmlcorrespondingauthor{Aleksandar Botev}{a.botev@cs.ucl.ac.uk}

\icmlkeywords{}

\vskip 0.3in
]

\printAffiliationsAndNotice{}

\begin{abstract}
We present an efficient block-diagonal approximation to the Gauss-Newton matrix for feedforward neural networks. Our resulting algorithm is competitive against state-of-the-art first-order optimisation methods, with sometimes significant improvement in optimisation performance. Unlike first-order methods, for which hyperparameter tuning of the optimisation parameters is often a laborious process, our approach can provide good performance even when used with default settings. A side result of our work is that for piecewise linear \transfers, the network objective function can have no differentiable local maxima, which may partially explain why such \transfers facilitate effective optimisation.
\end{abstract}

\section{Introduction}

First-order optimisation methods are the current workhorse for training neural networks.
They are easy to implement with modern automatic differentiation frameworks, scale to large models and datasets and can handle noisy gradients such as encountered in the typical mini-batch setting \citep{momentum_org, nag, adam, adagrad, adadelta}.
However, a suitable initial learning rate and decay schedule need to be selected in order for them to converge both rapidly and towards a good local minimum.
In practice, this usually means many separate runs of training with different settings of those hyperparameters, requiring access to either ample compute resources or plenty of time.
Furthermore, pure stochastic gradient descent often struggles to escape `valleys' in the error surface with largely varying magnitudes of curvature, as the first derivative does not capture this information \citep{saddle_points,hessian-free-rnn}.
Modern alternatives, such as ADAM \citep{adam}, combine the gradients at the current setting of the parameters with various heuristic estimates of the curvature from previous gradients.

Second-order methods, on the other hand, perform updates of the form $\delta=\fullh^{-1} \gradient$, where $\fullh$ is the Hessian or some approximation thereof and $\gradient$ is the gradient of the error function.
Using curvature information  enables such methods to make more progress per step than techniques relying solely on the gradient.
Unfortunately, for modern neural networks, explicit calculation and storage of the Hessian matrix is infeasible. Nevertheless, it is possible to efficiently calculate Hessian-vector products $\fullh \gradient$ by use of extended Automatic Differentiation \citep{fastmvp,rop}; the linear system $\gradient=\fullh v$ can then be solved for $v$, \emph{e.g.} by using conjugate gradients \citep{hessian-free-deep, hessian-free-rnn}. Whilst this can be effective, the number of iterations required makes this process uncompetitive against simpler first-order methods \citep{ontheimportance}.

In this work, we make the following contributions:
\begin{itemize}[style=unboxed,leftmargin=0cm,noitemsep,topsep=0pt]
\item[--] We develop a recursive block-diagonal approximation of the Hessian, where each block corresponds to the weights in a single feedforward layer.
These blocks are Kronecker factored and can be efficiently computed and inverted in a single backward pass.
\item[--] As a corollary of our recursive calculation of the Hessian, we note that for networks with piecewise linear \transfers the error surface has no differentiable strict local maxima.
\item[--] We discuss the relation of our method to KFAC \citep{kfac}, a block-diagonal approximation to the Fisher matrix.
KFAC is less generally applicable since it requires the network to define a probabilistic model on its output. Furthermore, for non-exponential family models, the Gauss-Newton and Fisher approaches are in general different. 
\item[--]  On three standard benchmarks we demonstrate that (without tuning) second-order methods perform competitively, even against well-tuned state-of-the-art first-order methods.
\end{itemize}


\section{Properties of the Hessian}
\label{background}

As a basis for our approximations to the Gauss-Newton matrix, we first describe how the diagonal Hessian blocks of feedforward networks can be recursively calculated. Full derivations are given in the supplementary material.

\subsection{Feedforward Neural Networks}
\label{sec:intro}

A feedforward neural network takes an input vector $\act_0=\x$ and produces an output vector $\pre_L$  on the final $(L^{th})$ layer of the network:
\beq
\pre_\lambda = \W_\lambda \act_{\lambda-1}; \quad \act_\lambda = \nonl_\lambda(\pre_\lambda) \ocm \; 1 \le \lambda < L 
\label{eq:nn}
\eeq
where $\pre_\lambda$ is the pre-activation in layer $\lambda$ and $\act_\lambda$ are the activation values; $\W_\lambda$ is the matrix of weights and $\nonl_\lambda$ the elementwise \transfer\footnote{The usual bias ${b}_\lambda$ in the equation for $\pre_\lambda$ is absorbed into $\W_\lambda$ by appending a unit term to every $\act_{\lambda-1}$.}. We define a loss $E(\pre_L,\y)$ between the output $\pre_L$ and a desired training output $\y$ (for example squared loss $(\pre_L-\y)^2$) which is a function of all parameters of the network $\params = \sq{\vect{\W_1}\trans, \vect{\W_2}\trans, \dots, \vect{\W_L}\trans}\trans$. For a training dataset with empirical distribution $p(\x,\y)$, the total error function is then defined as the expected loss $\bar{E}(\params)= \ave{E}_{p(\x,\y)}$. For simplicity we denote by $E(\params)$ the loss for a generic single datapoint $(\x,\y)$.


\subsection{The Hessian}

A central quantity of interest in this work is the parameter Hessian, $\fullh$, which has elements:
\beq
\sq{\fullh}_{ij} = \pdiff{\params_i}{\params_j}E(\params)
\eeq
The expected parameter Hessian is similarly given by the expectation of this equation.
To emphasise the distinction between the expected Hessian and the Hessian for a single datapoint $(\x,\y)$, we also refer to the single datapoint Hessian as the \emph{sample} Hessian.

\subsubsection{Block Diagonal Hessian}

The full Hessian, even of a moderately sized neural network, is computationally intractable due to the large number of parameters.
Nevertheless, as we will show, blocks of the sample Hessian can be computed efficiently. 
Each block corresponds to the second derivative with respect to the parameters $\W_\lambda$ of a single layer $\lambda$.
We focus on these blocks since the Hessian is in practice typically block-diagonal dominant \citep{kfac}. 


The gradient of the error function with respect to the weights of layer $\lambda$ can be computed by recursively applying the chain rule:\footnote{Generally we use a Greek letter to indicate a layer and a Roman letter to denote an element within a layer. We use either sub- or super-scripts wherever most notationally convenient and compact. }
\beq
\dE{\Wlab}
= \sum_i \deriv{\pre^\lambda_i}{\Wlab} \dE{\pre^\lambda_i}
=  \act^{\lambda-1}_b \dE{\pre^\lambda_a}
\eeq
Differentiating again we find that the sample Hessian for layer $\lambda$ is:
\begin{align}
\sq{\whess_\lambda}_{(a,b),(c,d)}
&\equiv \hE{\Wlab}{\Wlcd} \\
&= \act^{\lambda-1}_b \act^{\lambda-1}_d \sq{\xhess{\lambda}}_{a,c}
\label{eq:hess:derivation}
\end{align}
where  we define the \emph{\prea} Hessian for layer $\lambda$ as:
\beq
\label{eq:notations}
\sq{\xhess{\lambda}}_{a,b} = \hE{\pre^\lambda_a}{\pre^\lambda_b}
\eeq
We can re-express \eref{eq:hess:derivation} in matrix form for the sample Hessian of $\W_\lambda$:
\begin{equation}
\whess_\lambda = \hE{\vect{\ssind{\W}{\lambda}}}{\vect{\ssind{\W}{\lambda}}} = \br{\ssind{\act}{\lambda-1} {\ssind{\act}{\lambda-1}\trans}}\otimes \xhess{\lambda}
\label{eq:hess:kron}
\end{equation}
where  $\otimes$ denotes the Kronecker product\footnote{Using the notation $\{ \cdot \}_{i,j}$ as the $i,j$ matrix block, the Kronecker Product is defined as $
\{\m{A}\otimes\m{B}\}_{i,j} = a_{ij}\m{B}$.}.

\iftoggle{show_main_figs}{
\begin{figure*}[!ht]
	\begin{center}
		\begin{subfigure}[]{0.3\textwidth}
			\includegraphics[width=\textwidth]{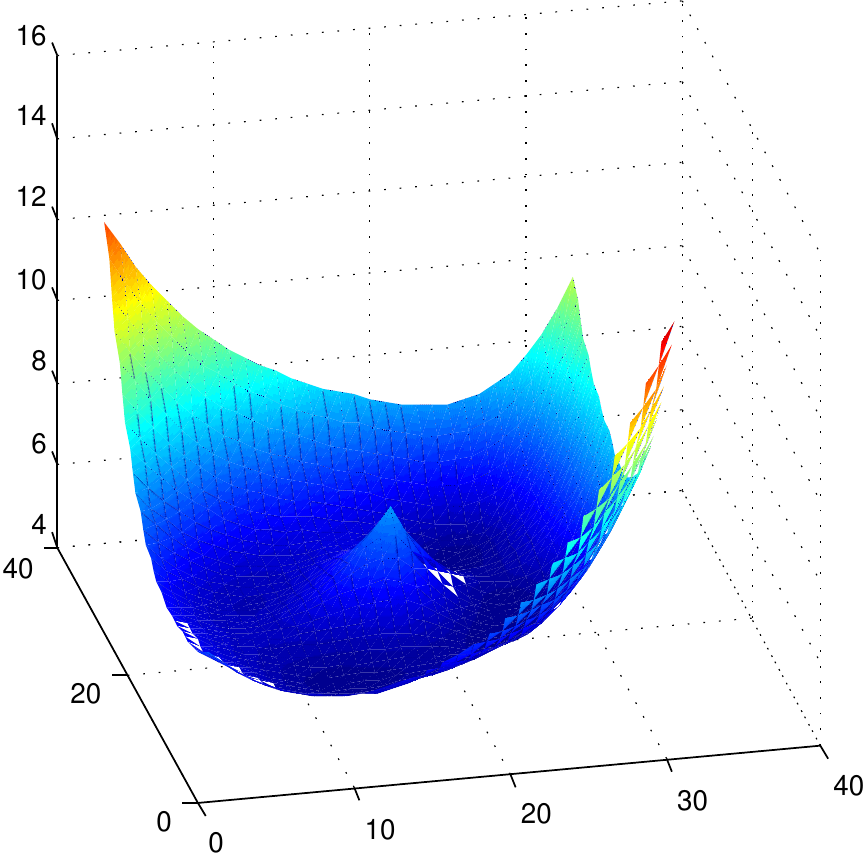}
			\caption{}
		\end{subfigure}%
		~
		\begin{subfigure}[]{0.3\textwidth}
			\includegraphics[width=\textwidth]{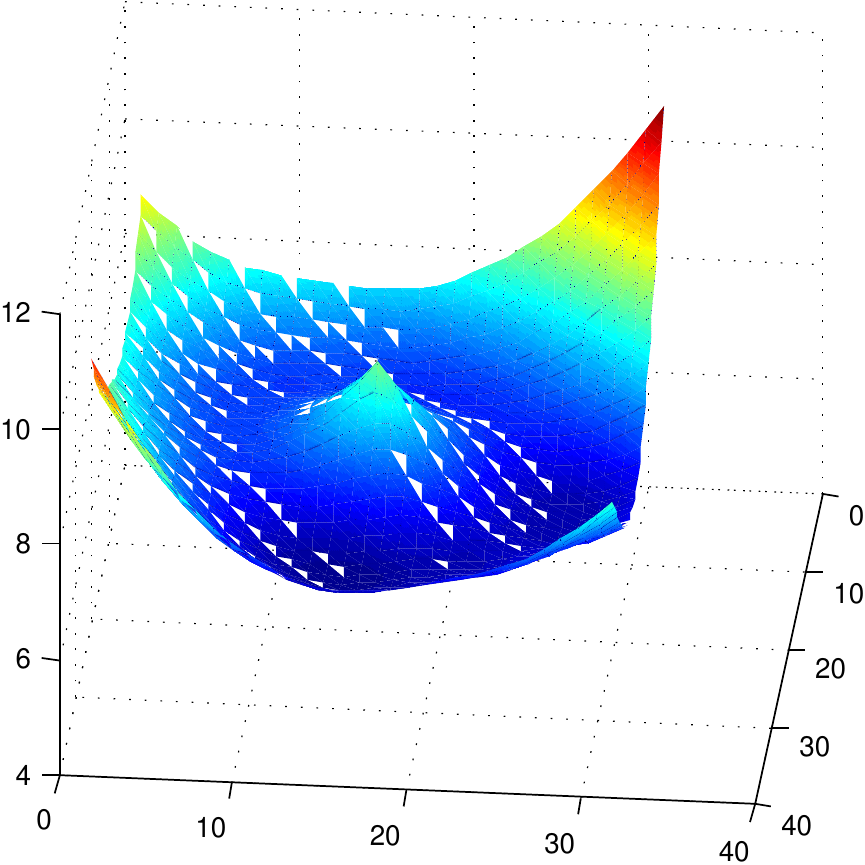}
			\caption{}
		\end{subfigure}%
		~
		\begin{subfigure}[]{0.3\textwidth}
			\includegraphics[width=\textwidth]{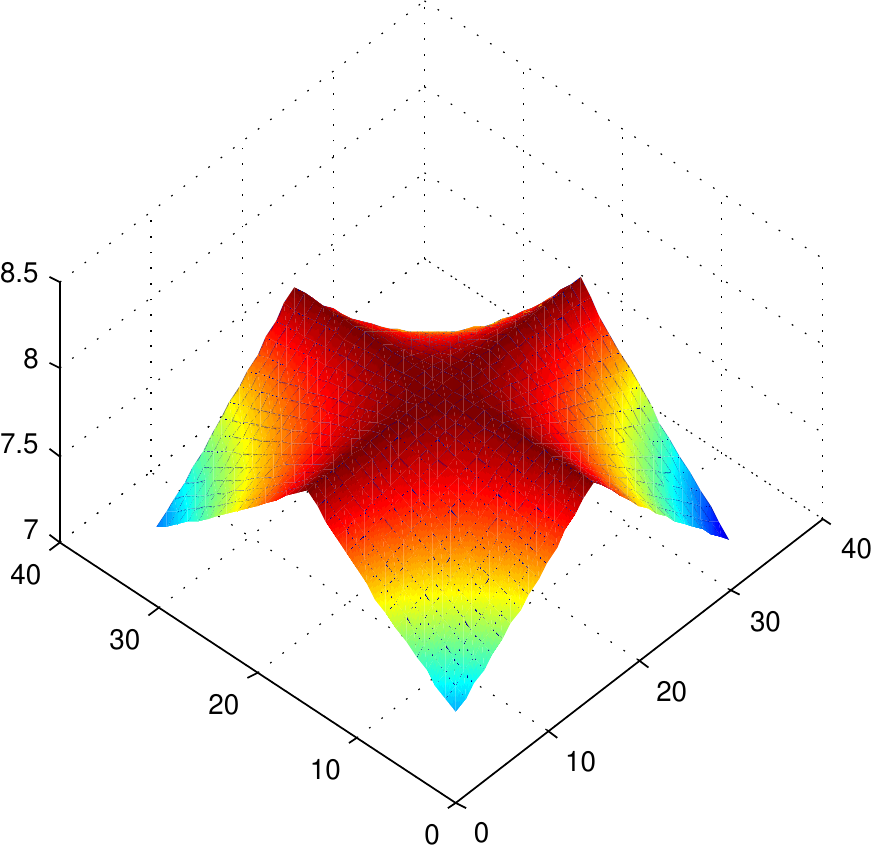}
			\caption{}
		\end{subfigure}
	\end{center} \vspace{-1\baselineskip}%
	\caption{Two layer network with ReLU and square loss. (a) The objective function $E$ as we vary $\W_1(x,y)$ along two randomly chosen direction matrices $U$ and $V$, giving $\W_1(x,y)=x U+yV$, $(x,y)\in\Real^2$. (b) $E$ as a function of two randomly chosen directions within $\W_2$. (c) $E$ for varying jointly $\W_1=xU$, $\W_2=yV$. The surfaces contain no smooth local maxima.}
	\label{fig:error:surface}
\end{figure*}
}

\subsubsection{Block Hessian Recursion}
\label{sec:block:hess}

In order to calculate the sample Hessian, we need to evaluate the \prea Hessian first. This can be computed recursively as (see \aref{app:block:hess}):
\beq
\xhess{\lambda} =  B_{\lambda} \W_{\lambda+1}\trans \xhess{\lambda+1} \W_{\lambda+1} B_{\lambda}  + D_{\lambda} 
\label{eq:hess:recursion}
\eeq
where we define the diagonal matrices:
\begin{align}
B_\lambda &= \diag{\nonl_\lambda'(\pre_\lambda)}\\
D_\lambda &= \diag{\nonl_\lambda''(\pre_\lambda) \dE{\act_\lambda}}
\label{eq:hes:d_def}
\end{align}
and $\nonl_\lambda'$ and $\nonl_\lambda''$ are the first and second derivatives of $\nonl_\lambda$ respectively. 

The recursion is initialised with $\xhess{L}$, which depends on the objective function $E(\params)$ and is easily calculated analytically for the usual objectives\footnote{For example for squared loss $(\y-\pre_L)^2/2$, the \prea Hessian is simply the identity matrix $\xhess{L}=I$.}.
Then we can simply apply the recursion  \eref{eq:hess:recursion} and compute the \prea Hessian for each layer using a single backward pass through the network.
A similar observation is given in \citep{pesky}, but restricted to the diagonal entries of the Hessian rather than the more general block-diagonal case.
Given the \prea Hessian, the Hessian of the parameters for a layer is given by \eref{eq:hess:kron}.
For more than a single datapoint, the recursion is applied per datapoint and the parameter Hessian is given by the average of the individual sample Hessians.

\subsubsection{No Differentiable Local Maxima}
\label{sec:pltf}

In recent years piecewise linear \transfers, such as the ReLU function$f(x) = \max(x, 0)$, have become popular\footnote{Note that, for piecewise linear $f$, $E$ is not necessarily piecewise linear in $\theta$.}.
%
Since the second derivative $f''$ of a piecewise linear function is zero everywhere, the matrices $D_\lambda$ in \eref{eq:hess:recursion} will be zero (away from non-differentiable points).

It follows that if $\xhess{L}$ is Positive Semi-Definite (PSD), which is the case for the most commonly used loss functions, the \prea matrices are PSD for every layer. 
%
A corollary is that if we fix all of the parameters of the network except for $\W_\lambda$ the objective function is locally convex with respect to $\W_\lambda$ wherever it is twice differentiable.
Hence, there can be no local maxima or saddle points of the objective with respect to the parameters within a layer\footnote{This excludes any pathological regions where the objective function has zero curvature.}.
Note that this does not imply that the objective is convex everywhere with respect to $\W_\lambda$ as the surface will contain ridges along which it is not differentiable, corresponding to boundary points where the \transfer changes regimes, see \fref{fig:error:surface}(c).

As the trace of the full Hessian $\fullh$ is the sum of the traces of the diagonal blocks, it must be non-negative and thus it is not possible for all eigenvalues to be simultaneously negative. This implies that for feedforward neural networks with piecewise linear \transfers there can be no differentiable local maxima - that is, outside of pathological constant regions, all maxima (with respect to the full parameter set $\theta$) must lie at the boundary points of the nonlinear activations and be `sharp', see \fref{fig:error:surface}. Additionally, for \transfers with zero gradient $f'=0$, $\xhess{\lambda}$ will have lower rank than  $\xhess{\lambda+1}$, reducing the curvature information propagating from the output layer back up the network. This suggests that it is advantageous to use piecewise linear \transfers with non-zero gradients, such as $\max(0.1x,x)$. 

We state and prove these results more formally in \aref{app:locmax}.

\section{Approximate Gauss-Newton Method\label{sec:gn:approx}}


Besides being intractable for large neural networks, the Hessian is not guaranteed to be PSD. A Newton update $\fullh^{-1}g$ could therefore lead to an increase in the error. A common PSD approximation to the Hessian is the Gauss-Newton (GN) matrix.
For an error $E(\pre^L(\params))$, the sample Hessian is given by:
\beq
\mhes{E}{\params_i}{\params_j} = \sum_k \dE{\pre^L_k} \mhes{\pre^L_k}{\params_i}{\params_j}+\sum_{k,l} \deriv{\pre^L_k}{\params_i} \hE{\pre^L_k}{\pre^L_l} \deriv{\pre^L_l}{\params_j}
\label{eq:gg:general}
\eeq
Assuming that $\xhess{L}$ is PSD, the GN method forms a PSD approximation by neglecting the first term in \eref{eq:gg:general}. 
This can be written in matrix notation as:
\beq
G \equiv {\Jth} \trans \xhess{L} \Jth
\label{eq:gn:sample}
\eeq
where $\J{\pre_L}{\params}$ is the Jacobian of the network outputs with respect to the parameters. The expected GN matrix is the average of \eref{eq:gn:sample} over the datapoints:
\beq
\bar{G} \equiv \ave{{\Jth} \trans \xhess{L} \Jth}_{p(\x,\y)}
\label{eq:gn:expect}
\eeq

Whilst \eref{eq:gn:expect} shows how to calculate the GN matrix exactly, in practice we cannot feasibly store the matrix in this raw form.  To proceed, similar to the Hessian, we will make a block diagonal approximation. As we will show, as for the Hessian itself, even a block diagonal approximation is computationally infeasible, and additional approximations are required. Before embarking on this sequence of approximations, we first show that the GN matrix can be expressed as the expectation of a Khatri-Rao product, \emph{i.e.} blocks of Kronecker products, corresponding to the weights of each layer. We will subsequently approximate the expectation of the Kronecker products as the product of the expectations of the factors, making the blocks efficiently invertible.

\subsection{The GN Matrix as a Khatri-Rao Product}

Using the definition of $\bar{G}$ in \eref{eq:gn:expect} and the chain rule, we can write the block of the matrix corresponding to the parameters in layers $\lambda$ and $\beta$ as:
\begin{align}
\bar{G}_{\lambda, \beta}
&= \ave{\J{\pre_\lambda}{\W_\lambda} \trans \J{\pre_L}{\pre_\lambda} \trans \xhess{L} \J{\pre_L}{\pre_\beta}  \J{\pre_\beta}{\W_\beta}}
\end{align}
where $[\J{\pre_L}{\pre_\lambda}]_{i,k}\equiv\deriv{\pre^L_k}{\pre^\lambda_i}$. Defining $\gn{\lambda, \beta}$ as the \prea  GN matrix between the $\lambda$ and $\beta$ layers' pre-activation vectors:
\beq
\gn{\lambda, \beta}  = \J{\pre_L}{\pre_\lambda} \trans \xhess{L} \J{\pre_L}{\pre_\beta}
\eeq
%
and using the fact that $\J{\pre_\lambda}{W_\lambda} = \act_{\lambda-1} \trans \otimes I$ we obtain
\beq
\bar{G}_{\lambda, \beta} = \ave{\br{ \act_{\lambda-1} \act_{\beta-1} \trans} \otimes \gn{\lambda, \beta}}
\label{eq:gn:blocks}
\eeq
We can therefore write the GN matrix as the expectation of the Khatri-Rao product: 
\beq
\bar{G} = \ave{ \qn{} \star \gn{}}
\label{eq:gn:khatrirao} 
\eeq
where the blocks of $\gn{}$ are the \prea GN matrices $\gn{\lambda,\beta}$ as defined in \eref{eq:gn:blocks}, and the blocks of $\qn{}$ are:
\beq
\qn{\lambda, \beta} \equiv{\act_{\lambda-1} \act_{\beta-1} \trans}
\eeq

\subsection{Approximating the GN Diagonal Blocks}

For simplicity, from here on we denote by $G_\lambda$ the diagonal blocks of the \emph{sample} GN matrix with respect to the weights of layer $\lambda$ (dropping the duplicate index). Similarly, we drop the index for the diagonal blocks $\qn{\lambda}$ and $\gn{\lambda}$ of the corresponding matrices in \eref{eq:gn:khatrirao}, giving more compactly:
\beq
G_\lambda =  \qn{\lambda} \otimes \gn{\lambda}
\eeq
The diagonal blocks of the expected GN $\bar{G}_\lambda$ are then given by $\ave{G_\lambda}$.
Computing this requires evaluating a block diagonal matrix for each datapoint and accumulating the result.
However, since the expectation of a Kronecker product is not necessarily Kronecker factored, one would need to explicitly store the whole matrix $\bar{G}_\lambda$ to perform this accumulation. With $D$ being the dimensionality of a layer, this matrix would have $O(D^4)$ elements. For $D$ of the order of $1000$, it would require several terabytes of memory to store $\bar{G}_\lambda$. As this is prohibitively large, we seek an approximation for the diagonal blocks that is both efficient to compute and store.
The approach we take is the {factorised} approximation:
\beq
\ave{G_\lambda}\approx \ave{\qn{\lambda}} \otimes \ave{\gn{\lambda}}\label{eq:gn:approx}
\eeq
Under this factorisation, the updates for each layer can be computed efficiently by solving a Kronecker product form linear system -- see the supplementary material.
%
The first factor $\ave{\qn{\lambda}}$ is simply the uncentered covariance of the activations:
\beq
\ave{\qn{\lambda}} = \frac{1}{N}A_{\lambda-1} A_{\lambda-1}\trans
\label{eq:q2}
\eeq
where the $n^{th}$ column of the $d \times n$ matrix $A_{\lambda-1}$ is the set of activations of layer $\lambda-1$ for datapoint $n$. The second factor $\ave{\gn{\lambda}}$, can be computed efficiently, as described below.

\subsection{The Pre-Activation Recursion}
\label{sec:block_recursion}

Analogously to the block diagonal \prea Hessian recursion \eref{eq:hess:recursion}, a similar recursion can be derived for the \prea GN matrix diagonal blocks:
%
\begin{align}
\gn{\lambda}
&= B_\lambda \W_{\lambda+1}\trans \gn{\lambda+1} \W_{\lambda+1} B_\lambda
\label{eq:gn:recursion}
\end{align}
where the recursion is initialised with the Hessian of the output $\xhess{L}$.


This highlights the close relationship between the \prea Hessian recursion and the \prea GN recursion. Inspecting \eref{eq:hess:recursion} and \eref{eq:gn:recursion} we notice that the only difference in the recursion stems from terms containing the diagonal matrices $D_\lambda$. 
From  \eref{eq:hess:kron} and \eref{eq:gn:blocks} it follows that in the case of piecewise linear \transfers, the diagonal blocks of the Hessian are equal to the diagonal blocks of the GN matrix\footnote{This holds only at points where the derivative exists.}.

Whilst this shows how to calculate the sample \prea GN blocks efficiently, from \eref{eq:gn:approx} we require the calculation of the \emph{expected} blocks $\ave{\gn{\lambda}}$.
In principle, the recursion could be applied for every data point.
However, this is impractical in terms of the computation time and a vectorised implementation would impose infeasible memory requirements.
Below, we show that when the number of outputs is small, it is in fact possible to efficiently compute the exact expected \prea GN matrix diagonals. For the case of a large number of outputs, we describe a further approximation to $\ave{\gn{\lambda}}$ in \sref{sec:rec}.

\subsection{Exact Low Rank Calculation of $\ave{\gn{\lambda}}$}
\label{sec:low_rank}
Many problems in classification and regression deal with a relatively small number of outputs. This implies that the rank $K$ of the output layer GN matrix  $\gn{L}$ is low. We use the square root representation:
%
%
\beq
\gn{\lambda} = \sum_{k=1}^K \Clk\Clk \trans 
\eeq
From \eref{eq:gn:recursion} we then obtain the recursion:
%
%
\beq
\Clk = B_{\lambda} \W_{\lambda+1}\trans \Clpk
\eeq
This allows us to calculate the expectation as:
\beq
\ave{\gn{\lambda}}
= \ave{\sum_k \Clk {\Clk} \trans}
=\frac{1}{N} \sum_k \tilde{C}_\lambda^k \br{\tilde{C}_\lambda^k} \trans
\eeq
where we stack the column vectors $\Clk$ for each datapoint into a matrix $\tilde{C}^\lambda_k$, analogous to \eref{eq:q2}.  Since we need to store only the vectors $\Clk$ per datapoint, this reduces the memory requirement to $K \times D \times N$; for small $K$ this is a computationally viable option.
We call this method Kronecker Factored Low Rank (KFLR).

\subsection{Recursive Approximation of $\ave{\gn{\lambda}}$\label{sec:rec}}

For higher dimensional outputs, \emph{e.g.} in autoencoders, rather than backpropagating a sample \prea GN matrix for every datapoint, we propose to simply pass the expected matrix through the network.
This yields the nested expectation approximation of \eref{eq:gn:recursion}:
\begin{align}
\ave{\gn{\lambda}}
&\approx \ave{B_\lambda \W_{\lambda+1} \trans \ave{\gn{\lambda+1}} \W_{\lambda+1} B_\lambda}
\end{align}
The recursion is initialised with the exact value $\ave{\gn{L}}$. 
The method will be referred to as Kronecker Factored Recursive Approximation (KFRA).

\section{Related Work}

Despite the prevalence of first-order methods for neural network optimisation, there has been considerable recent interest in developing practical second-order methods, which we briefly outline below.

\citet{hessian-free-deep} and  \citet{hessian-free-rnn} exploited the fact that full Gauss-Newton matrix-vector products can be computed efficiently using a form of automatic differentiation. This was used to approximately solve the linear system $\bar{G} \delta = \nabla f$ using conjugate gradients to find the parameter update $\delta$.
Despite making good progress on a per-iteration basis, having to run a conjugate gradient descent optimisation at every iteration proved too slow to compete with well-tuned first-order methods.


The closest related work to that presented here is the KFAC method \citep{kfac}, in which the Fisher matrix is used as the curvature matrix. This is based on the output $y$ of the network defining a conditional distribution $p_\params(y|x)$ on the observation $y$, with a loss defined as the KL-divergence between the empirical distribution $p(y|x)$ and the network output distribution.
The network weights are chosen to minimise the KL-divergence between the conditional output distribution and the data distribution.
For example, defining the network output as the mean of a fixed variance Gaussian or a Bernoulli/Categorical distribution yields the common squared error and cross-entropy objectives respectively.

Analogously to our work,  \citet{kfac} develop a block-diagonal approximation to the Fisher matrix. 
The Fisher matrix is another PSD approximation to the Hessian that is used in natural gradient descent \citep{natural_grad}. 
In general, the Fisher and GN matrices are different. However, for the case of $p_\params(y|x)$ defining an exponential family distribution, the Fisher and GN matrices are equivalent, see \aref{fisher_kfac}. As in our work, \citet{kfac} use a factorised approximation of the form \eref{eq:gn:approx}. However, they subsequently approximate the expected Fisher blocks by drawing Monte Carlo samples of the gradients from the conditional distribution defined by the neural network.  As a result, KFAC is always an approximation to the GN \prea matrix, whereas our method can provide an exact calculation of $\ave{\gn{}}$ in the low rank setting. See also \aref{app:kfac:differences} for differences between our KFRA approximation and KFAC.

More generally, our method does not require any probabilistic model interpretation and is therefore more widely applicable than KFAC.


\section{Experiments}

We performed experiments\footnote{Experiments were run on a workstation with a Titan Xp GPU and an Intel Xeon CPU E5-2620 v4 @ 2.10GHz.} training deep autoencoders on three standard grey-scale image datasets and classifying hand-written digits as odd or even.
The datasets are:

\begin{description}
\item[MNIST] consists of $60,000$ $28\times28$ images of hand-written digits. We used only the first $50,000$ images for training (since the remaining $10,000$ are usually used for validation). 
\item[CURVES] contains $20,000$ training images of size $28\times28$ pixels of simulated hand-drawn curves, created by choosing three random points in the $28\times28$ pixel plane (see the supplementary material of \citep{hinton_arch} for details).
\item[FACES] is an augmented version of the Olivetti faces dataset \citep{olivetti} with $10$ different images of $40$ people.
We follow \citep{hinton_arch} in creating a training set of $103,500$ images by choosing $414$ random pairs of rotation angles ($-90$ to $90$ degrees) and scaling factors ($1.4$ to $1.8$) for each of the $250$ images for the first $25$ people and then subsampling to $25\times25$ pixels.
\end{description}

We tested the performance of second-order against first-order methods and compared the quality of the different GN approximations.
In all experiments we report only the training error, as we are interested in the performance of the optimiser rather than how the models generalise.

When using second-order methods, it is important in practice to adjust the unmodified update $\delta$ in order to dampen potentially over-confident updates.
One of our central interests is to compare our approach against KFAC.
We therefore followed \citep{kfac} as closely as possible, introducing damping in an analogous way.
Details on the implementation are in \aref{app:imp}.
We emphasise that throughout all experiments we used the default damping parameter settings, with no tweaking required to obtain acceptable performance\footnote{Our damping parameters could be compared to the exponential decay parameters $\beta_1$ and $\beta_2$ in ADAM, which are typically left at their recommended default values.}.

Additionally, as a form of momentum for the second-order methods, we compared the use of a moving average with a factor of $0.9$ on the curvature matrices $\gn{\lambda}$ and $\qn{\lambda}$ to only estimating them from the current minibatch.
We did not find any benefit in using momentum on the updates themselves; on the contrary this made the optimisation unstable and required clipping the updates. We therefore do not include momentum on the updates in our results. 

All of the autoencoder architectures are inspired by \citep{hinton_arch}.
The layer sizes are $D$-$1000$-$500$-$250$-$30$-$250$-$500$-$1000$-$D$, where $D$ is the dimensionality of the input.
The grey-scale values are interpreted as the mean parameter of a Bernoulli distribution and the loss is the binary cross-entropy on CURVES and MNIST, and square error on FACES.

\subsection{Comparison to First-Order Methods}\label{sec:first:order:cmp}

\iftoggle{show_main_figs}{
{
\begin{figure*}[ht!]
    \centering
    \begin{subfigure}[b]{\linewidth}
        \centering
        \includegraphics[width=\textwidth]{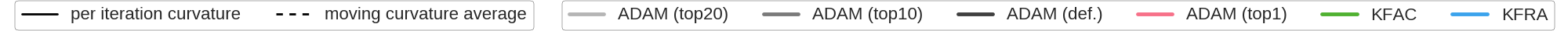}    
    \end{subfigure}%
    \\
    \begin{subfigure}[b]{0.33\linewidth}
        \centering
        \includegraphics[width=\textwidth]{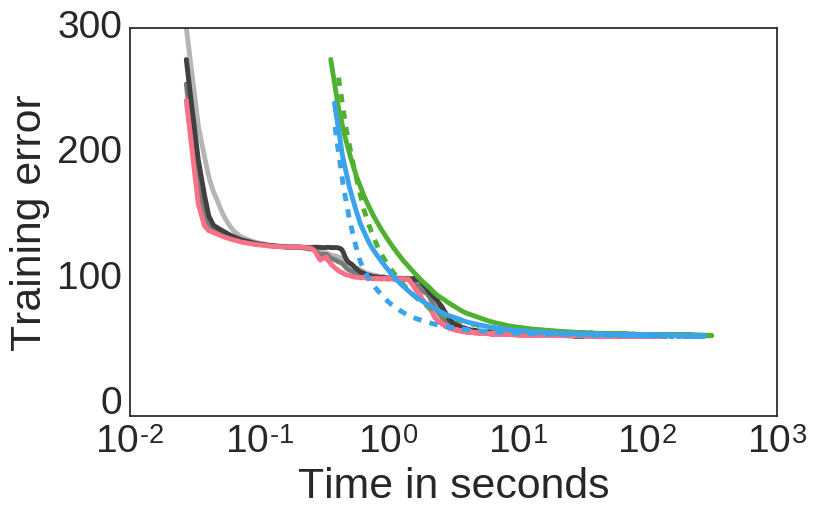}
        \label{fig:cmp:gpu:curves}
    \end{subfigure}%
    ~
    \begin{subfigure}[b]{0.33\linewidth}
        \centering
        \includegraphics[width=\textwidth]{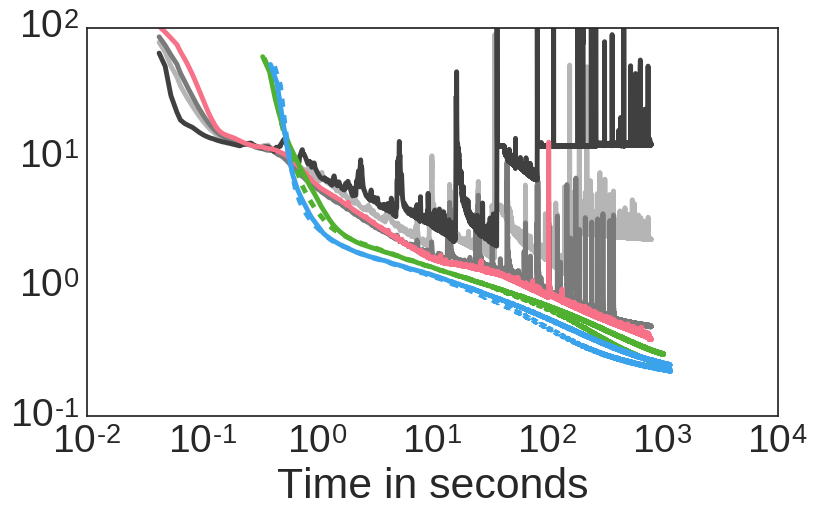}
        \label{fig:cmp:gpu:faces}
    \end{subfigure}%
    ~
    \begin{subfigure}[b]{0.33\linewidth}
        \centering
        \includegraphics[width=\textwidth]{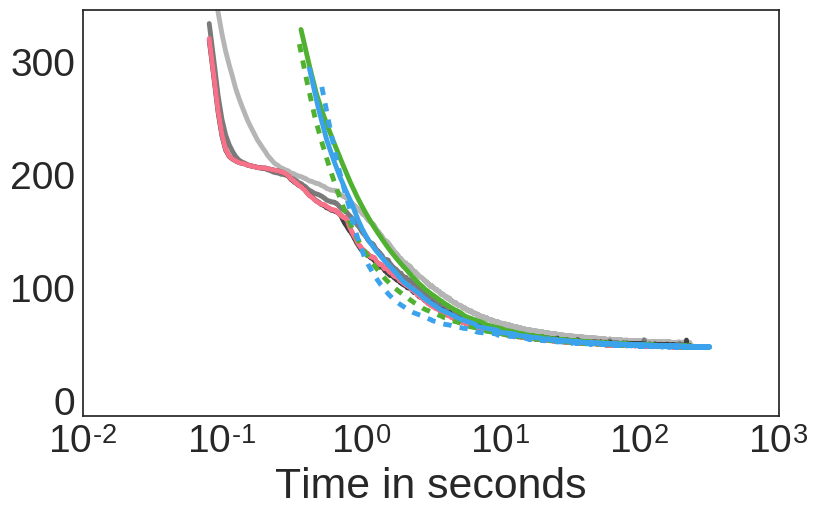}
        \label{fig:cmp:gpu:mnist}
    \end{subfigure} \vspace{-2\baselineskip}%
    \\%
    \begin{subfigure}[b]{0.33\linewidth}
        \centering
        \includegraphics[width=\textwidth]{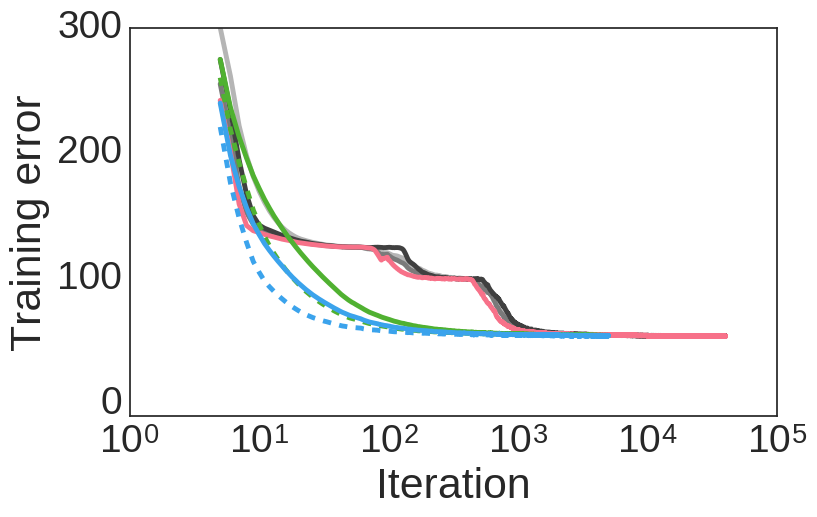}
        \caption{CURVES}
        \label{fig:cmp:iter:curves}
    \end{subfigure}%
    ~
    \begin{subfigure}[b]{0.33\linewidth}
        \centering
        \includegraphics[width=\textwidth]{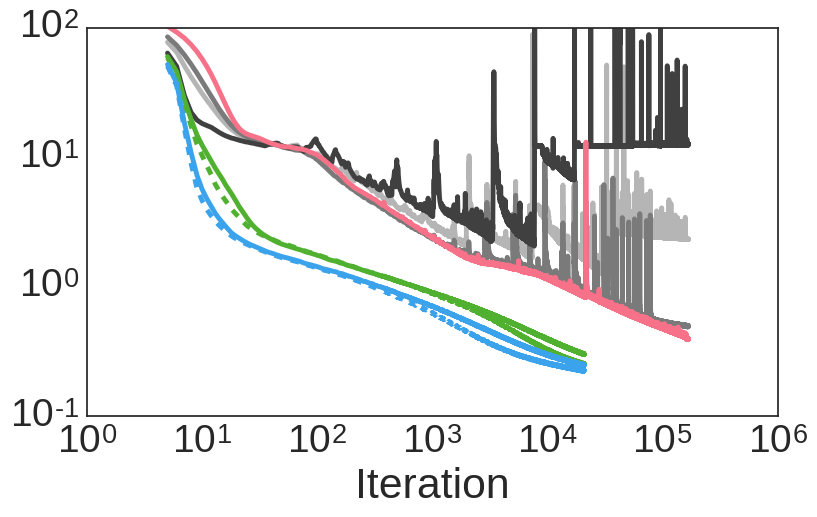}
        \caption{FACES}
        \label{fig:cmp:iter:faces}
    \end{subfigure}%
    ~
    \begin{subfigure}[b]{0.33\linewidth}
        \centering
        \includegraphics[width=\textwidth]{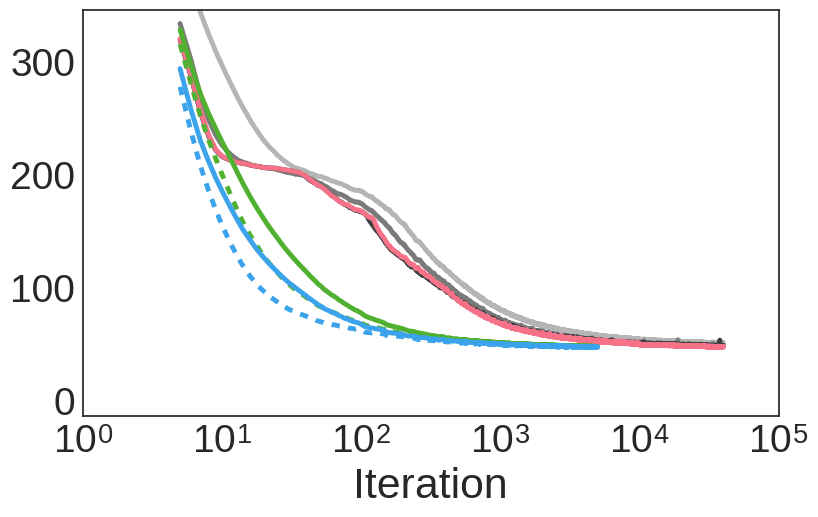}
        \caption{MNIST}
        \label{fig:cmp:iter:mnist}
    \end{subfigure} \vspace{-1.5\baselineskip}%
    \caption{Comparison of the objective function being optimised by KFRA, KFAC and ADAM on CURVES, FACES and MNIST. GPU benchmarks are in the first row, progress per update in the second. The dashed line indicates the use of momentum on the curvature matrix for the second-order methods. Errors are averaged using a sliding window of ten.}
    \label{fig:cmp}
\end{figure*}
}
}

We investigated the performance of both KFRA and KFAC compared to popular first-order methods.
Four of the most prevalent gradient-based optimisers were considered -- Stochastic Gradient Descent, Nesterov Accelerated Gradient, Momentum and ADAM \citep{adam}.
A common practice when using first-order methods is to decrease the learning rate throughout the training procedure.
For this reason we included an extra parameter $T$ -- the decay period -- to each of the methods, halving the learning rate every $T$ iterations.
To find the best first-order method, we ran a grid search over these two hyperarameters\footnote{We varied the learning rate from $2^{-6}$ to $2^{-13}$ at every power of $2$ and chose the decay period as one of $\{100\%, 50\%, 25\%, 12.5\%, 6.25\%\}$ of the number of updates.}.

Each first-order method was run for $40,000$ parameter updates for MNIST and CURVES and $160,000$ updates for FACES.
This resulted in a total of $35$ experiments and $1.4/5.6$ million updates for each dataset per method.
In contrast, the second-order methods did not require adjustment of any hyperparameters and were run for only $5,000/20,000$ updates, as they converged much faster\footnote{For fair comparison, all of the methods were implemented using Theano \citep{theano} and Lasagne \citep{lasagne}. }. 
For the first-order methods we found ADAM to outperform the others across the board and we consequently compared the second-order methods against ADAM only.

\fref{fig:cmp} shows the performance of the different optimisers on all three datasets.
We present progress both per parameter update, to demonstrate that the second-order optimisers effectively use the available curvature information, and per GPU wall clock time, as this is relevant when training a network in practice.
For ADAM, we display the performance using the default learning rate $10^{-3}$ as well as the top performing combination of learning rate and decay period.
To illustrate the sensitivity of ADAM to these hyperparameter settings (and how much can therefore be gained by parameter tuning) we also plot the average performance resulting from using the top 10 and top 20 settings.

Even after significantly tuning the ADAM learning rate and decay period, the second-order optimisers outperformed ADAM out-of-the-box across all three datasets.
In particular on the challenging FACES dataset, the optimisation was not only much faster when using second-order methods, but also more stable.
On this dataset, ADAM appears to be highly sensitive to the learning rate and in fact diverged when run with the default learning rate of $10^{-3}$.
In contrast to ADAM, the second-order optimisers did not get trapped in plateaus in which the error does not change significantly.

In comparison to KFAC, KFRA showed a noticeable speed-up in the optimisation both per-iteration and when measuring the wall clock time.
Their computational cost for each update is equivalent in practice, which we discuss in detail in \aref{app:diff_kfac_kfra}.
Thus, to validate that the advantage of KFRA over KFAC stems from the quality of its updates, we compared the alignment of the updates of each method with the exact Gauss-Newton update (using the slower Hessian-free approach; see \aref{app:updates} for the figures).
We found that KFRA tends to be more closely aligned with the exact Gauss-Newton update, which provides a possible explanation for its better performance.

\subsection{Non-Exponential Family Model}
\label{sec:exp:mix}

To compare our approximate Gauss-Newton method and KFAC in a setting where the Fisher and Gauss-Newton matrix are not equivalent, we designed an experiment in which the model distribution over $\y$ is not in the exponential family. The model is a mixture of two binary classifiers\footnote{In this context $\sigma(x) = (1 + \exp (-x))^{-1}$.}:
\beq
\begin{split}
p(\y|\pre_L) = \sigma(\pre^L_1) \sigma (\pre^L_2)^\y \sigma (-\pre^L_2)^{1-\y} + \\
(1 - \sigma(\pre^L_1)) \sigma (\pre^L_3)^\y \sigma (-\pre^L_3)^{1-\y}
\end{split}
\eeq
We used the same architecture as for the encoding layers of the autoencoders -- $D$-$1000$-$500$-$250$-$30$-$1$, where $D=784$ is the size of the input.
The task of the experiment was to classify MNIST digits as even or odd. 
Our choice was motivated by recent interest in neural network mixture models \citep{mixture1,mixture2,mixture3,mixture4}; our mixture model is also appropriate for testing the performance of KFLR. 
Training was run for $40,000$ updates for ADAM with a grid search as in \sref{sec:first:order:cmp}, and for $5,000$ updates for the second-order methods. 
The results are shown in \fref{fig:mnist:mix}.




For the CPU, both per iteration and wall clock time the second-order methods were faster than ADAM; on the GPU, however, ADAM was faster per wall clock time. The value of the objective function at the final parameter values was higher for second-order methods than for ADAM. However, it is important to keep in mind that all methods achieved a nearly perfect cross-entropy loss of around $10^{-8}$. 
When so close to the minimum we expect the gradients and curvature to be very small and potentially dominated by noise introduced from the mini-batch sampling.
Additionally, since the second-order methods invert the curvature, they are more prone to accumulating numerical errors than first-order methods, which may explain this behaviour close to a minimum.




Interestingly, KFAC performed almost identically to KFLR, despite the fact that KFLR computes the exact \prea Gauss-Newton matrix. 
This suggests that in the low-dimensional output setting, the benefits from using the exact low-rank calculation are diminished by the noise and the rather coarse factorised Kronecker approximation.

\section{Rank of the Empirical Curvature}

\iftoggle{show_main_figs}{
\begin{figure}[t!]
    \centering
    \includegraphics[width=\linewidth]{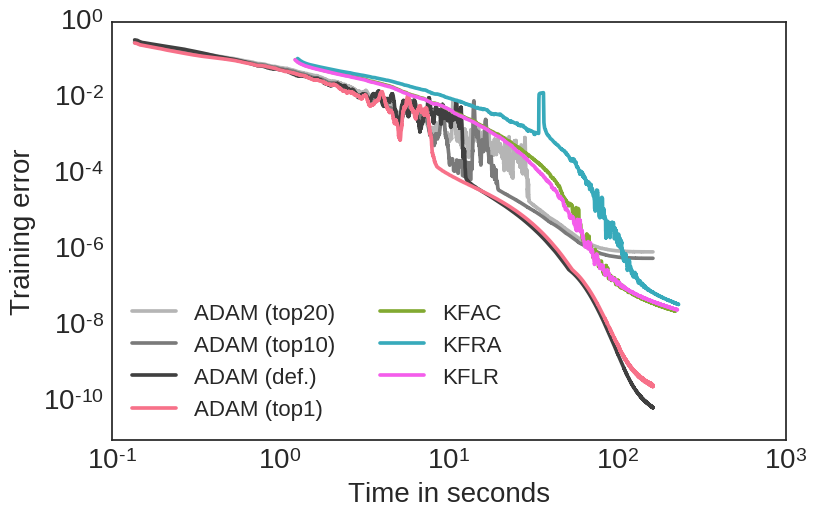} \vspace{-2\baselineskip}%
    \includegraphics[width=\linewidth]{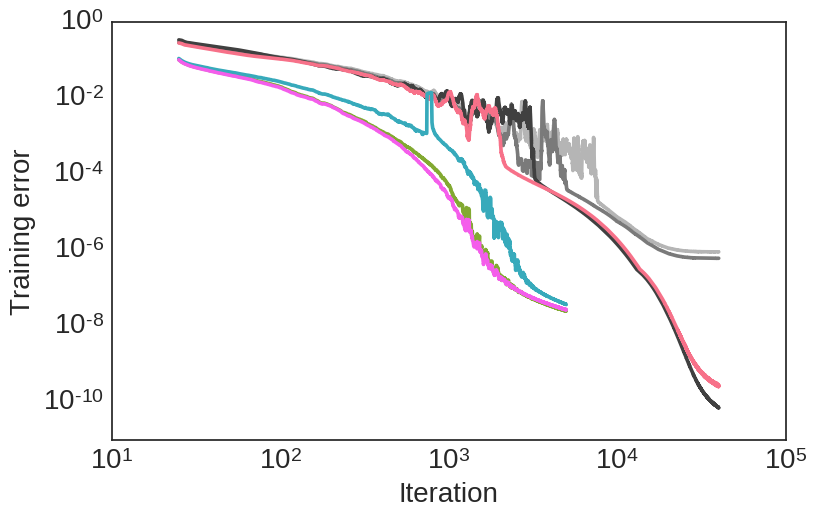} \vspace{0.1\baselineskip}
    \caption{Comparative optimisation performance on an MNIST binary mixture-classification model. We used momentum on the curvature matrix for all methods, as it stabilises convergence.}
    \label{fig:mnist:mix}
\end{figure}
}

The empirical success of second-order methods raises questions about the curvature of the error function of a neural network.
As we show in \aref{sec:gn_rank} the Monte Carlo Gauss-Newton matrix rank is upper bounded by the rank of the last layer Hessian times the size of the mini-batch.  More generally, the rank is upper bounded by the rank of $\xhess{L}$ times the size of the data set.
As modern neural networks commonly have millions of parameters, the exact Gauss-Newton matrix is usually severely under-determined.
This implies that the curvature will be zero in many directions.
This phenomenon is particularly pronounced for the binary classifier in \sref{sec:exp:mix}, where the rank of the output layer Hessian is one.

We can draw a parallel between the curvature being zero and standard techniques where the maximum likelihood problem is under-determined for small data sets.
This explains why damping is so important in such situations, and its role goes beyond simply improving the numerical stability of the algorithm. Our results suggest that, whilst in practice the Gauss-Newton matrix provides curvature only in a limited parameter subspace, this still provides enough information to allow for relatively large parameter updates compared to gradient descent, see \fref{fig:cmp}.

\section{Conclusion}

We presented a derivation of the block-diagonal structure of the Hessian matrix arising in feedforward neural networks. This leads directly to the interesting conclusion that for networks with piecewise linear \transfers and convex loss the objective has no differentiable local maxima. Furthermore, with respect to the parameters of a single layer, the objective has no differentiable saddle points. This may provide some partial insight into the success of such transfer functions in practice.


Since the Hessian is not guaranteed to be positive semi-definite, two common alternative curvature measures are the Fisher matrix and the Gauss-Newton matrix. Unfortunately, both are computationally infeasible and, similar to \citet{kfac}, we therefore used a block diagonal approximation, followed by a factorised Kronecker approximation. 
Despite parallels with the Fisher approach, formally the two methods are different.
Only in the special case of exponential family models are the Fisher and Gauss-Newton matrices equivalent; however, even for this case, the subsequent approximations used in the Fisher approach \citep{kfac} differ from ours.
Indeed, we showed that for problems in which the network has a small number of outputs no additional approximations are required.
Even on models where the Fisher and Gauss-Newton matrices are equivalent, our experimental results suggest that our KFRA approximation performs marginally better than KFAC.
As we demonstrated, this is possibly due to the updates of KFRA being more closely aligned with the exact Gauss-Newton updates than those of KFAC.

Over the past decade first-order methods have been predominant for Deep Learning. 
Second-order methods, such as Gauss-Newton, have largely been dismissed because of their seemingly prohibitive computational cost and potential instability introduced by using mini-batches.
Our results on comparing both the Fisher and Gauss-Newton approximate methods, in line with \cite{kfac}, confirm that second-order methods can perform admirably against even well-tuned state-of-the-art first-order approaches, while not requiring any hyperparameter tuning.

In terms of wall clock time on a CPU, in our experiments, the second-order approaches converged to the minimum significantly more quickly than state-of-the-art first-order methods.
When training on a GPU (as is common in practice), we also found that second-order methods can perform well, although the improvement over first-order methods was more marginal.
However, since second-order methods are much faster per update, there is the potential to further improve their practical utility by speeding up the most expensive computations, specifically solving linear systems on parallel compute devices.

\section*{Acknowledgements}

We thank the reviewers for their valuable feedback and suggestions.
We also thank Raza Habib, Harshil Shah and James Townsend for their feedback on earlier drafts of this paper.
Finally, we are grateful to James Martens for helpful discussions on the implementation of KFAC.

\bibliographystyle{icml2017}

\newpage
\appendix
\onecolumn

%
%
%
%

\section{Derivation of the Block-Diagonal Hessian Recursion}
\label{app:block:hess}

The diagonal blocks of the \prea Hessian of a feedforward neural network can be related to each other via the recursion in \eref{eq:hess:recursion}. Starting from its definition in \eref{eq:notations} we can derive the recursion:
\begin{equation}
\begin{split}
\xhess{\lambda} &=  \deriv{}{\pre^\lambda_b} \dE{\pre^\lambda_a} = \deriv{}{\pre^\lambda_b} \sum_i \dE{\pre^{\lambda+1}_i} \deriv{\pre^{\lambda+1}_i}{\pre^\lambda_a} = \sum_i \deriv{}{\pre^\lambda_b} \br{ \dE{\pre^{\lambda+1}_i} \deriv{\pre^{\lambda+1}_i}{\act^{\lambda}_a} \deriv{\act^{\lambda}_a}{\pre^{\lambda}_a} }\\
&= \sum_i \W^{\lambda+1}_{i,a} \deriv{}{\pre^\lambda_b} \br{ \dE{\pre^{\lambda+1}_i} \deriv{\act^{\lambda}_a}{\pre^{\lambda}_a} } = 
\sum_i \W^{\lambda+1}_{i,a} \br{\deriv{\act^{\lambda}_a}{\pre^{\lambda}_a} \hE{\pre^\lambda_b}{\pre^{\lambda+1}_i } + \dE{\pre^{\lambda+1}_i} \mhes{\act^\lambda_a}{\pre^\lambda_a}{\pre^\lambda_b}} \\
&= \sum_i \W^{\lambda+1}_{i,a} \br{\deriv{\act^{\lambda}_a}{\pre^{\lambda}_a} \sum_j \hE{\pre^{\lambda+1}_j}{\pre^{\lambda+1}_i} \deriv{\pre^{\lambda+1}_j}{\pre^\lambda_b} + \dE{\pre^{\lambda+1}_i} \delta_{a,b} \dderiv{\act^\lambda_a}{\pre^\lambda_a}} \\
&= \delta_{a,b} \dderiv{\act^\lambda_a}{\pre^\lambda_a} \br{\sum_i \W^{\lambda+1}_{i,a} \dE{\pre^{\lambda+1}_i}} + \sum_{i,j}  \W^{\lambda+1}_{i,a} \deriv{\act^{\lambda}_a}{\pre^{\lambda}_a} \hE{\pre^{\lambda+1}_j}{\pre^{\lambda+1}_i} \W^{\lambda+1}_{j,b} \deriv{\act^{\lambda}_b}{\pre^{\lambda}_b} \\
&= \delta_{a,b} \dderiv{\act^\lambda_a}{\pre^\lambda_a} \dE{\act^{\lambda}_a} + \sum_{i,j}  \W^{\lambda+1}_{i,a} \deriv{\act^{\lambda}_a}{\pre^{\lambda}_a} \hE{\pre^{\lambda+1}_j}{\pre^{\lambda+1}_i}  \deriv{\act^{\lambda}_b}{\pre^{\lambda}_b} \W^{\lambda+1}_{j,b}
\end{split}
\end{equation}
Hence the \prea Hessian can be written in matrix notation as
\beq
\xhess{\lambda} =  B_{\lambda} \W_{\lambda+1}\trans \xhess{\lambda+1} \W_{\lambda+1} B_{\lambda}  + D_{\lambda} 
\eeq
where we define the diagonal matrices 
\begin{align}
\sq{B_\lambda}_{a,a'} &= \delta_{a,a'}\deriv{\act^\lambda_a}{\pre^\lambda_a} = \delta_{a,a'}\nonl'(\pre^\lambda_a)\\
\sq{D_\lambda}_{a,a'} &= \delta_{a,a'}\frac{\partial^2 \act^\lambda_a}{\partial {\pre^\lambda_a}^2} \dE{x^\lambda_a} = \delta_{a,a'}\nonl''(\pre^\lambda_a) \dE{x^\lambda_a}
\end{align}
$\nonl'$ and $\nonl''$ are the first and second derivatives of the \transfer $\nonl$ respectively.

Note that this is a special case of a more general recursion for calculating a Hessian \citep{higher_order_ad}.

\section{Implementation Details\label{app:imp}}

Second-order optimisation methods are based on finding some positive semi-definite quadratic approximation to the function of interest around the current value of the parameters. For the rest of the appendix we define $\hat{f}$ to be a local quadratic approximation to $f$ given a positive semi-definite curvature matrix $C$:
\beq	
\label{impl::approx}
f(\params + \delta) \approx  f(\params) + \delta \trans \nabla_\params f + \frac{1}{2} \delta \trans C \delta = \hat{f}(\delta; C) 
\eeq 
The curvature matrix depends on the specific optimisation method and will often be only an estimate. For notational simplicity, the dependence of $\hat{f}$ on $\params$ is omitted.
Setting $C$ to the true Hessian matrix of $f$ would make $\hat{f}$ the exact second-order Taylor expansion of the function around $\params$. 
However, when $f$ is a nonlinear function, the Hessian can be indefinite, which leads to an ill-conditioned quadratic approximation $\hat{f}$.
For this reason, $C$ is usually chosen to be positive-semi definite by construction, such as the Gauss-Newton or the Fisher matrix.
In the experiments discussed in the paper, $C$ can be either the full Gauss-Newton matrix $\bar{G}$, obtained from running Conjugate Gradient as in \citep{hessian-free-deep}, or a block diagonal approximation to it, denoted by $\widetilde{G}$. 
The analysis below is independent of whether this approximation is based on KFLR, KFRA, KFAC or if it is the exact block-diagonal part of $\bar{G}$, hence there will be no reference to a specific approximation.

Damping plays an important role in second-order optimisation methods.
It can improve the numerical stability of the quadratic approximation and can also be related to trust region methods.
Hence, we will introduce two damping coefficients - one for $\bar{G}$ and one for $\widetilde{G}$.
In practice, an additional weight decay term is often applied to neural network models.
As a result and following the presentation in \citep{kfac}, a total of three extra parameters are introduced:

\begin{itemize}
	\item A $L_2$ regularisation on $\params$ with weight $\frac{\eta}{2}$, which implies an additive diagonal term to both $\bar{G}$ and $\widetilde{G}$
	\item A damping parameter $\tau$ added to the full Gauss-Newton matrix $\bar{G}$
	\item A separate damping parameter $\gamma$ added to the approximation matrix $\widetilde{G}$
\end{itemize} 
Subsequently, for notational convenience, we define $\bar{C} = \bar{G} + (\tau + \eta) I$, which is the curvature matrix obtained when using the full Gauss-Newton matrix in the quadratic approximation \eref{impl::approx}. Similarly, $\widetilde{C} = \widetilde{G} + (\gamma + \eta) I$ is the curvature matrix obtained when using any of the block-diagonal approximations.



\subsection{Inverting the Approximate Curvature Matrix}\label{app:inverse}

The Gauss-Newton method calculates its step direction by multiplying the gradient with the inverse of the curvature matrix, in this case $\widetilde{C}$.
This gives the unmodified step direction:
\beq
\widetilde{\delta} = \widetilde{C}^{-1} \nabla_\params f
\label{eq:impl:invert}
\eeq
Since $\widetilde{C}$ is a block diagonal matrix (each block corresponds to the parameters of a single layer) the problem naturally factorises to solving $L$ independent linear systems:
\beq
\wtl{\delta} = \wtl{C}^{-1} \nabla_{\W_\lambda} f
\label{eq:impl:lamb_invert}
\eeq
For all of the approximate methods under consideration -- KFLR, KFRA and KFAC -- the diagonal blocks $\wtl{G}$ have a Kronecker factored form $\wtl{Q} \otimes \wtl{\gn{}}$, where $\wtl{Q} = \expect{Q_\lambda}$ and $\wtl{\gn{}}$ denotes the approximation to $\expect{\gn{\lambda}}$ obtained from the method of choice. Setting $k = (\eta + \gamma)$ implies:
\beq
\wtl{C} = \wtl{Q} \otimes \wtl{\gn{}} + k I \otimes I
\label{eq:impl:k}
\eeq

The exact calculation of \eref{eq:impl:lamb_invert} given the structural form of $\wtl{C}$ requires the eigen decomposition of both matrices $\wtl{Q}$ and $\wtl{\gn{}}$, see \citep{kfac}. 
However, the well known Kronecker identity $(A \otimes B)^{-1} \vect{V} = A^{-1} V B^{-1}$ motivates the following approximation:
\beq
\label{eq:approx:pi}
\wtl{Q} \otimes \wtl{\gn{}} + k I \otimes I \approx \br{\wtl{Q} + \omega \sqrt{k} I} \otimes \br{\wtl{\gn{}} + \omega^{-1} \sqrt{k} I}
\eeq
The optimal setting of $\omega$ can be found analytically by bounding the norm of the approximation's residual, namely:
\beq
\begin{split}
R(\omega) &= \wtl{Q} \otimes \wtl{\gn{}} + k I \otimes I - \br{\wtl{Q} + \omega \sqrt{k} I} \otimes \br{\wtl{\gn{}} + \omega^{-1} \sqrt{k} I} \\
&= - \omega^{-1} \sqrt{k} \wtl{Q} \otimes I - \omega \sqrt{k} \wtl{\gn{}} \otimes I \\
||R(\pi)|| &\le \omega^{-1} \sqrt{k} || \wtl{Q} \otimes I || + \omega \sqrt{k} ||\wtl{\gn{}} \otimes I|| 
\end{split}
\eeq

Minimising the right hand side with respect to $\omega$ gives the solution
\beq
\omega = \sqrt{\frac{||\widetilde{Q}_\lambda \otimes I||}{||I \otimes \widetilde{\gn{\lambda}}||}}
\eeq

The choice of the norm is arbitrary, but  for comparative purposes with previous work on KFAC, we use the trace norm in all of our experiments.
Importantly, this approach is computationally cheaper as it requires solving only two linear systems per layer, compared to an eigen decomposition and four matrix-matrix multiplications for the exact calculation.
Alternatively, one can consider this approach as a special form of damping for Kronecker factored matrices.


\subsection{Choosing the Step Size}

The approximate block diagonal Gauss-Newton update can be significantly different from the full Gauss-Newton update. 
It is therefore important, in practice, to choose an appropriate step size, especially in cases where the curvature matrices are estimated from mini-batches rather than the full dataset. 
The step size is calculated based on the work in \citep{kfac}, using the quadratic approximation $\hat{f}(\delta; \bar{C})$ from \eref{impl::approx}, induced by the full Gauss-Newton matrix.
Given the initial step direction $\widetilde{\delta}$ from \eref{eq:impl:invert} a line search is performed along that direction and the resulting optimal step size is used for the final update.
\beq
\alpha_* = \argmin_\alpha \hat{f}(\alpha \widetilde{\delta}; \bar{C}) = \argmin_\alpha f(\params) + \alpha \widetilde{\delta}\trans \nabla f + \frac{1}{2} \alpha^2 \widetilde{\delta}\trans \bar{C} \widetilde{\delta}
\eeq
This can be readily solved as
\beq
\delta_* = \alpha_* \widetilde{\delta} = - \frac{\widetilde{\delta}\trans \nabla f }{\widetilde{\delta}\trans \bar{C} \widetilde{\delta}} \widetilde{\delta} 
\label{eq:impl:deltastar}
\eeq
where
\beq
\widetilde{\delta}\trans \bar{C} \widetilde{\delta} = \widetilde{\delta}\trans \bar{G} \widetilde{\delta} + (\tau + \eta) \widetilde{\delta}\trans \widetilde{\delta}
\eeq
The term $\widetilde{\delta}\trans \bar{G} {\delta}$ can be calculated efficiently (without explicitly evaluating $\bar{G}$) using the R-operator \citep{rop}. 
The final update of the approximate GN method is $\delta_*$. 
Notably, the damping parameter $\gamma$ affects the resulting update direction $\widetilde{\delta}$, whilst $\tau$ affects only the step size along that direction.

\subsection{Adaptive Damping}

\subsubsection{$\tau$}

In order to be able to correctly adapt $\tau$ to the current curvature of the objective we use a Levenberg-Marquardt heuristic based on the reduction ratio $\rho$ defined as
\begin{equation}
\rho = \frac{f(\params + \delta_*) - f(\params)}{\hat{f}(\delta_*; \bar{C}) - \hat{f}(0; \bar{C})}
\end{equation}
This quantity measures how well the quadratic approximation matches the true function. 
When $\rho < 1$ it means that the true function $f$ has a lower value at $\params + \delta_*$  (and thus the quadratic underestimates the curvature), while in the other case the quadratic overestimates the curvature. 
The Levenberg-Marquardt method introduces the parameter $\omega_\tau < 1$. When  $\rho > 0.75$ the $\tau$ parameter is multiplied by $\omega_\tau$, when  $\rho < 0.75$ it is divided by $\omega_\tau$.
In order for this to not introduce a significant computational overhead (as it requires an additional evaluation of the function -- $f(\params + \delta_*)$) we adapt $\tau$ only every $T_\tau$ iterations. 
For all experiments we  used $\omega_\tau = 0.95^{T_\tau}$ and $T_\tau = 5$.

\subsubsection{$\gamma$}

The role of $\gamma$ is to regularise the approximation of the quadratic function $\hat{f}(\delta, \widetilde{C})$ induced by the approximate Gauss-Newton to that induced by the full Gauss-Newton $\hat{f}(\delta, \bar{C})$.
This is in contrast to $\tau$, which regularises the quality of the latter approximation to the true function.
$\gamma$ can be related to how well the approximate curvature matrix $\widetilde{C}$ reflects the full Gauss-Newton matrix. 
The main reason for having two parameters is because we have two levels of approximations, and each parameter independently affects each one of them:
\beq
f(\params + \delta) \stackrel{\tau}{\approx} \hat{f}(\delta; \bar{C}) \stackrel{\gamma}{\approx} \hat{f}(\delta; \widetilde{C}) 
\eeq

The parameter $\gamma$ is updated greedily.
Every $T_\gamma$ iterations the algorithm computes the update $\delta_*$ for each of $\{\omega_\gamma \gamma, \gamma, \omega_\gamma^{-1} \gamma, \}$ and some scaling factor $\omega_\gamma < 1$. 
From these three values the one that minimises $\hat{f}(\delta; \bar{C})$ is selected. 
Similar to the previous section, we use $ \omega_\gamma = 0.95^{T_\gamma}$ and $T_\gamma = 20$ across all experiments.

\subsection{Parameter Averaging}
Compared with stochastic first-order methods (for example stochastic gradient descent), stochastic second-order methods do not exhibit any implicit averaging.  
To address this, \citet{kfac} introduce a separate value $\widehat{\params}_t$ which tracks the moving average of the parameter values $\params_t$  used for training:
\beq
\widehat{\params}_t = \beta_t \widehat{\params}_{t-1} + (1 - \beta_t) \params_t
\eeq
Importantly, $\widehat{\params}_t$ has no effect or overhead on training as it is not used for the updates on $\params_t$. 
The extra parameter $\beta_t$ is chosen such that in the initial stage when $t$ is small, $\widehat{\params}_t$ is the exact average of the first $t$ parameter values of $\params$:
\beq
\beta_t = \min(0.95, 1 - 1/t)
\eeq

Another factor playing an important role in stochastic second-order methods is the mini-batch size $m$.
In \citet{kfac}, the authors concluded that because of the high signal to noise ratio that arises close to the minimum, in practice one should use increasingly larger batch sizes for KFAC as the optimisation proceeds. 
However, our work does not focus on this aspect and all of the experiments are conducted using a fixed batch size of $1000$.

\section{The Fisher Matrix and KFAC}\label{sec:kfac}

\subsection{The Fisher Matrix}

For a general probability distribution $p_\params(x)$ parametrised by $\params$, the Fisher matrix can be expressed in two equivalent forms \citep{insights}:
\beq
\begin{split}
	\bar{F} &= \expect{\loggrad{\params}{\x} \loggrad{\params}{\x} \trans }_{p_\params(\x)} \\
	&= -\expect{\nabla\loggrad{}{\x}}_{p_\params(\x)} 
\end{split}
\label{eq:fisher:def}
\eeq
By construction the Fisher matrix is positive semi-definite. Using the Fisher matrix in place of the Hessian to form the parameter update $\bar{F}^{-1}g$ is known as Natural Gradient \citep{natural_grad}.

In the neural network setting, the model specifies a conditional distribution $\pmodel_\params(\y|\x)$, and the Fisher matrix is given by
\begin{equation}
\begin{split}
\bar{F} &= \expect{\loggrad{\params}{\x,\y} \loggrad{\params}{\x,\y} \trans }_{\pmodel_\params(\x,\y)} \\
&= \expect{\loggradc{\params} \loggradc{\params} \trans}_{\pmodel_\params(\x,\y)}
\end{split}
\end{equation}
Using the chain rule $\loggradc{\params} = \Jth \trans \loggradc{\pre_L}$ and defining
\beq
\fn{L} \equiv \loggradc{\pre_L} \loggradc{\pre_L} \trans 
\eeq
the Fisher can be calculated as:
\beq
\begin{split}
	\bar{F} &= \ave{\Jth \trans \loggradc{\pre_L} \loggradc{\pre_L} \trans \Jth}_{\pmodel_\params (\x,\y)} \\
	&=  \expect{\Jth \trans \fn{L} \Jth }_{\pmodel_\params (\x,\y)}
\end{split}
\eeq
The equation is reminiscent of \eref{eq:gn:sample} and in \aref{equivalence} we discuss the conditions under which the Fisher and the Gauss-Newton matrix are indeed equivalent.

\subsection{Equivalence between the Gauss-Newton and Fisher Matrices \label{equivalence}}
The expected Gauss-Newton matrix is given by
\beq
\bar{G} = \expect{\Jth \trans \xhess{L} \Jth}_{p(\x,\y)} = \expect{\Jth \trans \expect{\xhess{L}}_{p(\y|\x)} \Jth}_{p(\x)}
\eeq
Using that $\ave{\fn{L}} = \ave{\xhess{L}}$ which follows from \eref{eq:fisher:def} and the fact that the Jacobian $\Jth$ is independent of $\y$, the Fisher matrix can be expressed as:
\begin{equation}
\begin{split}
\bar{F} &= \expect{\Jth \trans \fn{L} \Jth}_{\pmodel_\params (\x,\y)} = \expect{\Jth \trans \expect{\fn{L}}_{\pmodel_\params (\y|\x)} \Jth}_{p(\x)}=
\expect{\Jth \trans \expect{\xhess{L}}_{\pmodel_\params (\y|\x)} \Jth}_{p(\x)}
\end{split}
\label{eq:fisher:equiv}
\end{equation}
Hence the Fisher and Gauss-Newton matrices matrices are equivalent when $\expect{\xhess{L}}_{p(\y|\x)} = \expect{\xhess{L}}_{\pmodel_\params (\y|\x)}$. 
Since the model distribution $\pmodel_\params (\y|\x)$ and the true data distribution $p(\y|\x)$ are not equal, a sufficient condition for the expectations to be equal is $\xhess{L}$ being independent of $\y$. 
Although this might appear restrictive, if $\pre_L$ parametrises the natural parameters of an exponential family distribution this independence holds \citep{insights}. To show this, consider
\beq
\log p(\y|\x, \params) = \log h(\y) + T(\y)\trans \eta(\x, \params) - \log Z(\x, \params) = \log h(\y) + T(\y)\trans \pre_L - \log Z(\pre_L)
\eeq
where $h$ is the base measure, $T$ is the sufficient statistic, $Z$ is the partition function and $\eta$ are the natural parameters. Taking the gradient of the log-likelihood with respect to $\pre_L$
\beq
\nabla_{\pre_L} \log p(\y|\x, \params) = T(\y) - \nabla_{\pre_L} \log Z(\pre_L) 
\eeq
Assuming that the objective is the negative log-likelihood as in \sref{sec:intro} and differentiating again
\beq
\xhess{L} = \nabla\nabla_{\pre_L}  \log Z(\pre_L)
\label{eq:gn:indep}
\eeq

which is indeed independent of $\y$. 
This demonstrates that in many practical supervised learning problems in Machine Learning, the Gauss-Newton and Natural Gradient methods are equivalent.


The parameter update for these approaches is then given by computing $\bar{G}^{-1}g$ or $\bar{F}^{-1}g$, where $g$ is the gradient of the objective with respect to all parameters.
However, the size of the linear systems is prohibitively large in the case of neural networks, thus it is computationally infeasible to solve them exactly.
As shown in \citep{fastmvp}, matrix-vector products with $\bar{G}$ can be computed efficiently using the R-operator \citep{rop}. 
The method does not need to compute $\bar{G}$ explicitly, at the cost of approximately twice the computation time of a standard backward pass. 
This makes it suitable for iterative methods such as conjugate gradient for approximately solving the linear system.
The resulting method is called `Hessian-free', with promising results in deep feedforward and recurrent neural networks \citep{hessian-free-deep, hessian-free-rnn}.
Nevertheless, the convergence of the conjugate gradient step may require many iterations by itself, which can significantly increase the computation overhead compared to standard methods. 
As a result, this approach can have worse performance per clock time compared to a well-tuned first-order method \citep{ontheimportance}. 
This motivates the usage of approximate Gauss-Newton methods instead.

\subsection{The Fisher Approximation to $\ave{\gn{\lambda}}$ and KFAC \label{fisher_kfac}}

The key idea in this approach is to use the fact that the Fisher matrix is an expectation of the outer product of gradients and that it is equal to the Gauss-Newton matrix (\sref{equivalence}). 
This is independent of whether the Gauss-Newton matrix is with respect to $\W_\lambda$ or $\pre_\lambda$, so we can write the \prea Gauss-Newton as

\begin{align}
\ave{\gn{\lambda}}_{p(\x,\y)}
&= \ave{\J{\pre_L}{\pre_\lambda} \trans \xhess{L} \J{\pre_L}{\pre_\lambda}}_{p(\x)} \\
&= \ave{\J{\pre_L}{\pre_\lambda} \trans \ave{\fn{L}}_{p_\params(\y|\x)} \J{\pre_L}{\pre_\lambda}}_{p(\x)}\\
&= \ave{\J{\pre_L}{\pre_\lambda} \trans \fn{L} \J{\pre_L}{\pre_\lambda}}_{p_\params(\x, \y)} \\
&= \ave{\J{\pre_L}{\pre_\lambda} \trans \loggradc{\pre_L} \loggradc{\pre_L} \trans \J{\pre_L}{\pre_\lambda}}_{p_\params(\x, \y)} \\
&=\ave{\loggradc{\pre_\lambda} \loggradc{\pre_\lambda} \trans }_{p_\params(\x,\y)}
\label{eq:fisher:aprox}
\end{align}

where the first equality follows from \eref{eq:gn:indep} and the second one from \eref{eq:fisher:equiv} in the supplement.

We stress here that the resulting expectation is over the model distribution $p_\params(\x,\y)$ and not the data distribution $p(\x,\y)$. 
In order to approximate \eref{eq:fisher:aprox} the method proceeds by taking Monte Carlo samples of the gradients from the model conditional distribution $p_\params(\y|\x)$. 

The KFAC approximation presented in \citep{kfac} is analogous to the above approach, but it is derived in a different way. 
The authors directly focus on the parameter Fisher matrix. 
Using the fact that  $\J{\pre_\lambda}{\W_\lambda} = \act_{\lambda-1} \trans \otimes I$ and $\J{\pre_\lambda}{W_\lambda} \trans v = \act_{\lambda-1} \otimes v$, the blocks of the Fisher matrix become:
\begin{align}
\sq{\bar{F}}_{\lambda,\beta}
&= \ave{\nabla_{\W_\lambda} \log p_\params(\y|\x) \nabla_{W_\beta} \log p_\params(\y|\x) \trans}_{p_\params(\x,\y)} \\
&= \ave{\br{\act_{\lambda-1} \otimes \nabla_{\pre_\lambda} \log p_\params(\y|\x) } \br{\act_{\beta-1} \otimes \nabla_{\pre_\beta} \log p_\params(\y|\x)} \trans}_{p_\params(\x,\y)} \\
&= \ave{\br{\act_{\lambda-1} \act_{\beta-1} \trans} \otimes \br{ \nabla_{\pre_\lambda} \log p_\params(\y|\x) \nabla_{\pre_\beta} \log p_\params(\y|\x)\trans }}_{p_\params(\x,\y)}
\end{align}
This equation is equivalent to our result in \eref{eq:gn:blocks} \footnote{Under the condition that the Fisher and Gauss-Newton matrices are equivalent, see \sref{equivalence}}. In \citep{kfac} the authors similarly approximate the expectation of the Kronecker products by the product of the individual expectations, which makes the second term equal to the \prea Gauss-Newton as in \eref{eq:fisher:aprox}.

\subsection{Differences between KFAC and KFRA}\label{app:kfac:differences}
\label{app:diff_kfac_kfra}

It is useful to understand the computational complexity of both KFAC and KFRA and the implications of the approximations.
In order to not make any additional assumptions about the underlying hardware or mode (serial or parallel) of execution, we denote with $O_{mm}(m,n,p)$ the complexity of a matrix matrix multiplication of an $m \times n$ and $n \times p$ matrices and with $O_{el}(m,n)$ the complexity of an elementwise multiplication of two matrices of size $m \times n$.

\begin{description}
	\item[KFRA] We need to backpropagate the matrix $\ave{\gn{\lambda}}$ of size $D_\lambda \times D_\lambda$, where $D_\lambda$ is the dimensionality of the layer. For each layer, this requires two matrix-matrix multiplications with $W_\lambda$ and single element wise multiplication (this is due to $A_\lambda$ being diagonal, which allows for such a simplification). The overall complexity of the procedure is $2 O_{mm}(D_\lambda, D_\lambda,  D_{\lambda-1}) + O_{el}(D_{\lambda-1}, D_{\lambda-1})$.
	
	\item[KFAC] We need to draw samples from $p_\params(y|x)$ for each datapoint $x$ and backprogate the corresponding gradients through the network (this is in addition to the gradients of the objective function). This requires backpropagating a matrix of size $D_{\lambda-1} \times NS$, where $S$ denotes the number of samples taken per datapoint. Per layer, the method requires also two matrix-matrix multiplications (one with $W_\lambda$ and the outer product of $C_\lambda^s$) and a single element wise multiplication. The overall complexity of the procedure is $O_{mm}(NS, D_\lambda,  D_{\lambda-1}) + O_{el}(NS, D_{\lambda-1}) + O_{mm}(D_{\lambda-1}, NS, D_{\lambda-1})$.
	
\end{description}
There are several observations which we deem important. Firstly, if $N = 1$, KFRA is no longer an approximate method, but computes the exact $\gn{\lambda}$ matrix. Secondly, if $S = 1$ and $D_{\lambda} \sim N$ then the two methods have similar computational complexity. If we assume that the complexity scales linearly in $S$, in the extreme case of $S = N$ and $D_{\lambda} \sim N$, it is possible to execute KFRA independently for each datapoint providing the exact value $\gn{\lambda}$ for the same computational cost, while KFAC would nevertheless still be an approximation. 

\section{The Rank of the Monte Carlo Gauss-Newton}
\label{sec:gn_rank}
Using the definition of the sample Gauss-Newton matrix in \eref{eq:gn:sample} we can infer that its rank is bounded by the rank of $\xhess{L}$:
\beq
G \equiv {\Jth} \trans \xhess{L} \Jth \Rightarrow \mathbf{rank}(G) \leq \mathbf{rank}(\xhess{L})
\eeq
This does not provide any bound on the rank of the ``true" Gauss-Newton matrix, which is an expectation of the above.
However, for any practical algorithm which approximates the expectations via $N$ Monte Carlo samples as:
\beq
\bar{G} = \expect{G} \approx \frac{1}{N} \sum_n G_n 
\eeq
it provides a bound on the rank of the resulting matrix.
Using the sub-additive property of the rank operator, it follows that $\mathbf{rank}(\frac{1}{N} \sum_n G_n ) \le \mathbf{rank}(\xhess{L}) N$.
Similarly, the approximate Fisher matrix computed in KFAC will have a rank bounded by $NS$, where $S$ is the number of samples taken per data point (usually one). 
This provides an explanation for the results in \sref{sec:exp:mix} for binary classification, since the last layer output in this problem is a scalar, thus its rank is 1.
Hence, both the Gauss-Newton and the Fisher for a mini-batch have a rank bounded by the mini-batch size $N$.
This leads to the conclusion that in such a situation the curvature information provided from a Monte-Carlo estimate is not sufficient to render the approximate Gauss-Newton methods competitive against well-tuned first order methods, although we observe that in the initial stages they are still better.
In some of the more recent works on KFAC the authors use momentum terms in conjunction with the second-order updates or do not rescale by the full Gauss-Newton.
This leaves space for exploration and more in depth research on developing techniques that can robustly and consistently improve the performance of second-order methods for models with a small number of outputs and small batch sizes.

\section{Absence of Smooth Local Maxima for Piecewise Linear Transfer Functions}\label{app:locmax}

In order to show that the Hessian of a neural network with piecewise linear \transfers can have no differentiable strict local maxima, we first establish that all of its diagonal blocks are positive semi-definite.

\begin{lemm}
\label{eq:lemma1}
Consider a neural network as defined in \eref{eq:nn}.
If the second derivative of all \transfers $f_\lambda$ for $1 \leq \lambda \leq L$ is zero where defined, and if the Hessian of the error function w.r.t. the outputs of the network is positive semi-definite, then all blocks
\beq
\fullh_\lambda = \frac{\partial^2 E}{\partial \vect{W_\lambda} \partial \vect{W_\lambda}}
\eeq
on the diagonal of the Hessian --- corresponding to the weights $\W_\lambda$ of single layer --- are positive semi-definite.
\label{lemma1}
\end{lemm}

\begin{proof}

By the definition in \eref{eq:hes:d_def} $D^\lambda_{i,j} = \delta_{i,j} \nonl '' (h^\lambda_{i,j})$. 
From the assumption of the lemma,  $\nonl '' = 0$ for all layers, hence $\forall \lambda \quad D_\lambda = 0$. 
Using the recursive equation \eref{eq:hess:recursion} we can analyze the quadratic form $v \trans \xhess{\lambda} v$:
\beq
\begin{split}
\xhess{\lambda} &= B_{\lambda} \W_{\lambda+1}\trans \xhess{\lambda+1} \W_{\lambda+1} B_{\lambda}  + D_{\lambda} \\
&=  \left( \W_{\lambda+1} B_{\lambda} \right) \trans \xhess{\lambda+1} \left(\W_{\lambda+1} B_{\lambda} \right)
\end{split}
\eeq
where we used the fact that by definition $B_\lambda$ is a diagonal matrix, thus it is symmetric and $B_\lambda = B_\lambda \trans$. 
Defining 
\beq
\tilde{v} =  \W_{\lambda+1} B_{\lambda} v
\eeq
yields
\beq
\begin{split}
v \trans \xhess{\lambda} v &= \left( \W_{\lambda+1} B_{\lambda} v \right) \trans \xhess{\lambda+1} \left( \W_{\lambda+1} B_{\lambda} v \right) \\
&= \tilde{v} \trans \xhess{\lambda+1} \tilde{v} 
\end{split}
\eeq
hence
\beq
\xhess{\lambda+1} \geq 0 \Rightarrow \xhess{\lambda} \geq 0 
\eeq

It follows by induction that if $\xhess{L}$ is positive semi-definite, all of the \prea Hessian matrices are positive semi-definite as well.

Using the proof that the blocks $\fullh_\lambda$ can be written as a Kronecker product in \eref{eq:hess:kron}, we can analyze the quadratic form of the Hessian block diagonals:
\beq
\begin{split}
\vect{V} \trans \fullh_{\lambda} \vect{V} &= \vect{V} \trans \left[\left(\act_\lambda \act_\lambda \trans \right) \otimes \xhess{\lambda} \right] \vect{V} \\
&= \vect{V} \trans \vect{\xhess{\lambda} V \act_\lambda \act_\lambda \trans} \\
&= \trace \left( V \trans \xhess{\lambda} V \act_\lambda \act_\lambda \trans \right) \\
&= \trace \left( \act_\lambda \trans V \trans \xhess{\lambda} V \act_\lambda \right) \\
&= \left(V \act_\lambda \right) \trans \xhess{\lambda} \left(V \act_\lambda \right) \\
\xhess{\lambda} \geq 0 \Rightarrow \fullh_\lambda \geq 0 
\end{split}
\eeq
The second line follows from the well known identity $(A \otimes B) \vect{V} = \vect{B V A \trans}$. 
Similarly, the third line follows from the fact that $\vect{A} \trans \vect{B} = \trace \left( A \trans B \right)$.
The fourth line uses the fact that $\trace \left( A B \right) = \trace \left( B A \right)$ when the product $AB$ is a square matrix. This concludes the proof of the lemma.

\end{proof}

This lemma has two implications:
\begin{itemize}
    \item If we fix the weights of all layers but one, the error function becomes \emph{locally} convex, wherever the second derivatives of all \transfers in the network are defined.
    \item The error function can have no differentiable strict local maxima.
\end{itemize}

We formalise the proof of the second proposition below:

\begin{lemm}
Under the same assumptions as in Lemma~\ref{lemma1}, the overall objective function $E$ has no differentiable strict local maxima with respect to the parameters $\params$ of the network. 
\label{lemma2}
\end{lemm}

\begin{proof}
For a point to be a strict local maximum, all eigenvalues of the Hessian at this location would need to be simultaneously negative.
However, as the trace of a matrix is equal to the sum of the eigenvalues it is sufficient to prove that $\trace{(\fullh)} \geq 0$.

The trace of the full Hessian matrix is equal to the sum of the diagonal elements, so it is also equal to the sum of the traces of the diagonal blocks.
Under the assumptions in Lemma~\ref{lemma1}, we showed that all of the diagonal blocks are positive semi-definite, hence their traces are non-negative.
It immediately follows that:
\beq
\trace{(\fullh)} = \sum_{\lambda=1}^L \trace{(\fullh_\lambda)} \ge 0
\eeq
Therefore, it is impossible for all eigenvalues of the Hessian to be simultaneously negative.
As a corollary it follows that all strict local maxima must lie in the non-differentiable boundary  points of the nonlinear \transfers.
\end{proof}

\newpage





\section{Additional Figures}

\subsection{CPU Benchmarks}

\iftoggle{show_appendix_figs}{
\begin{figure*}[h!]
    \centering
    \begin{subfigure}[b]{0.6\textwidth}
        \includegraphics[width=0.95\linewidth]{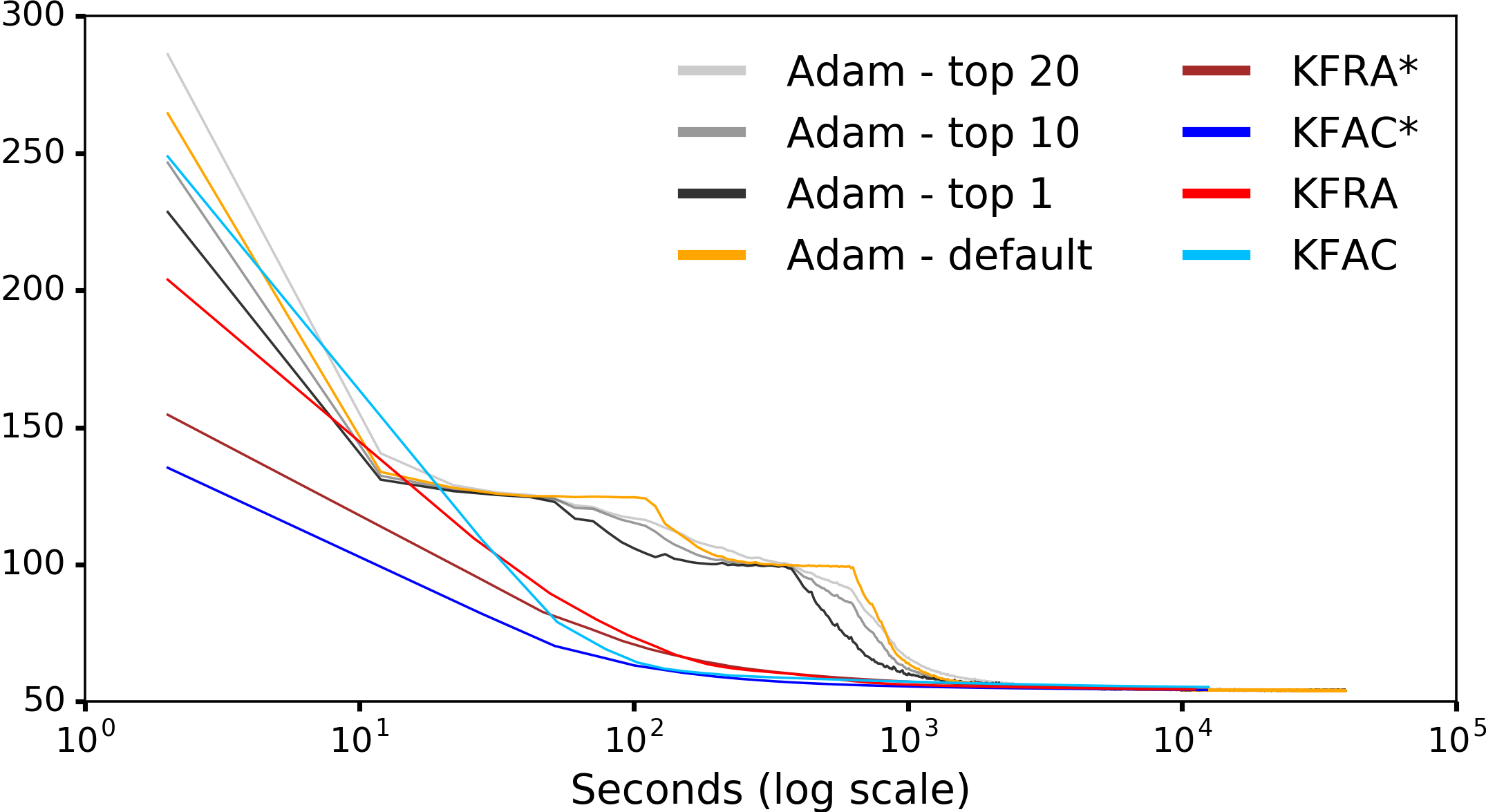}
        \caption{CURVES}
    \end{subfigure}%
    \\
    \begin{subfigure}[b]{0.6\textwidth}
        \includegraphics[width=0.95\linewidth]{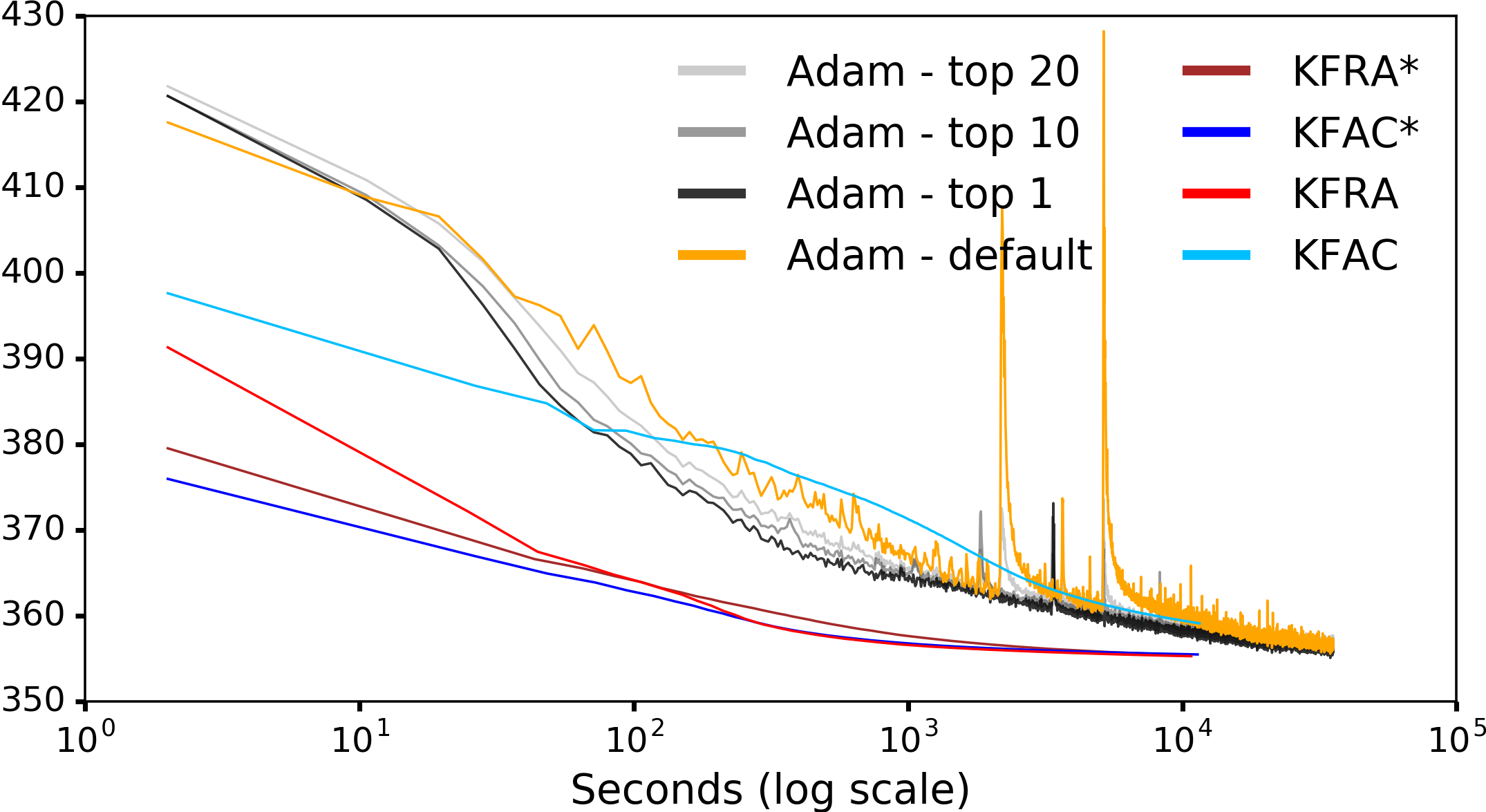}
        \caption{FACES}
    \end{subfigure}%
    \\
    \begin{subfigure}[b]{0.6\textwidth}
        \includegraphics[width=0.95\linewidth]{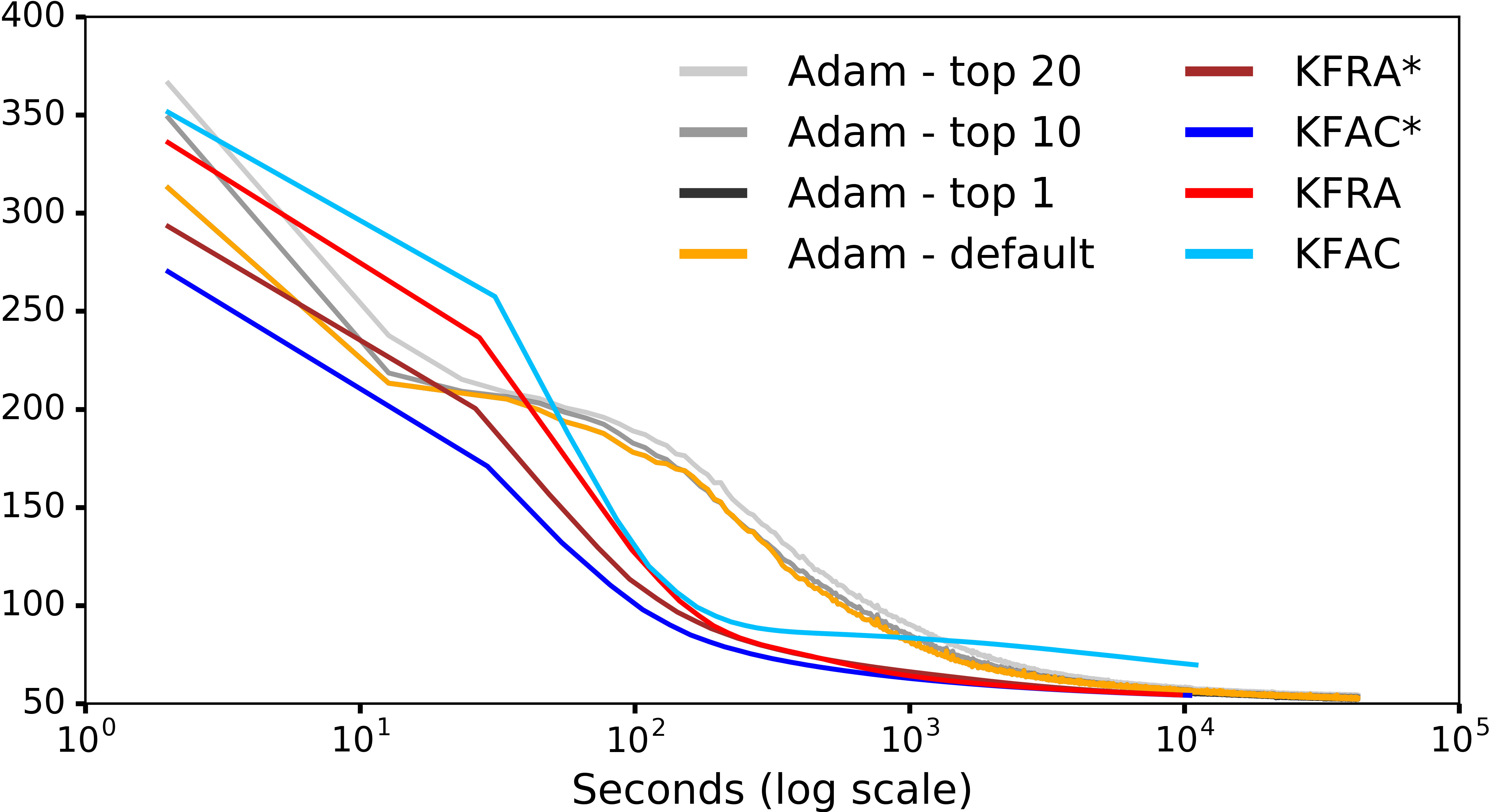}
        \caption{MNIST}
    \end{subfigure}
    \caption{Optimisation performance on the CPU. These timings are obtained with a previous implementation in Arrayfire, different to the one used for the figures in the main text. For the second-order methods, the asterisk indicates the use of the approximate inversion as described in \sref{app:inverse}. The error function on all three datasets is binary cross-entropy.}
    \label{fig:cmp:cpu}
\end{figure*}
}

\clearpage

\subsection{Comparison of the Alignment of the Approximate Updates with the Gauss-Newton Update}
\label{app:updates}

\iftoggle{show_appendix_figs}{
\begin{figure*}[h!]
    \centering
    \begin{subfigure}[b]{0.5\textwidth}
        \includegraphics[width=0.95\linewidth]{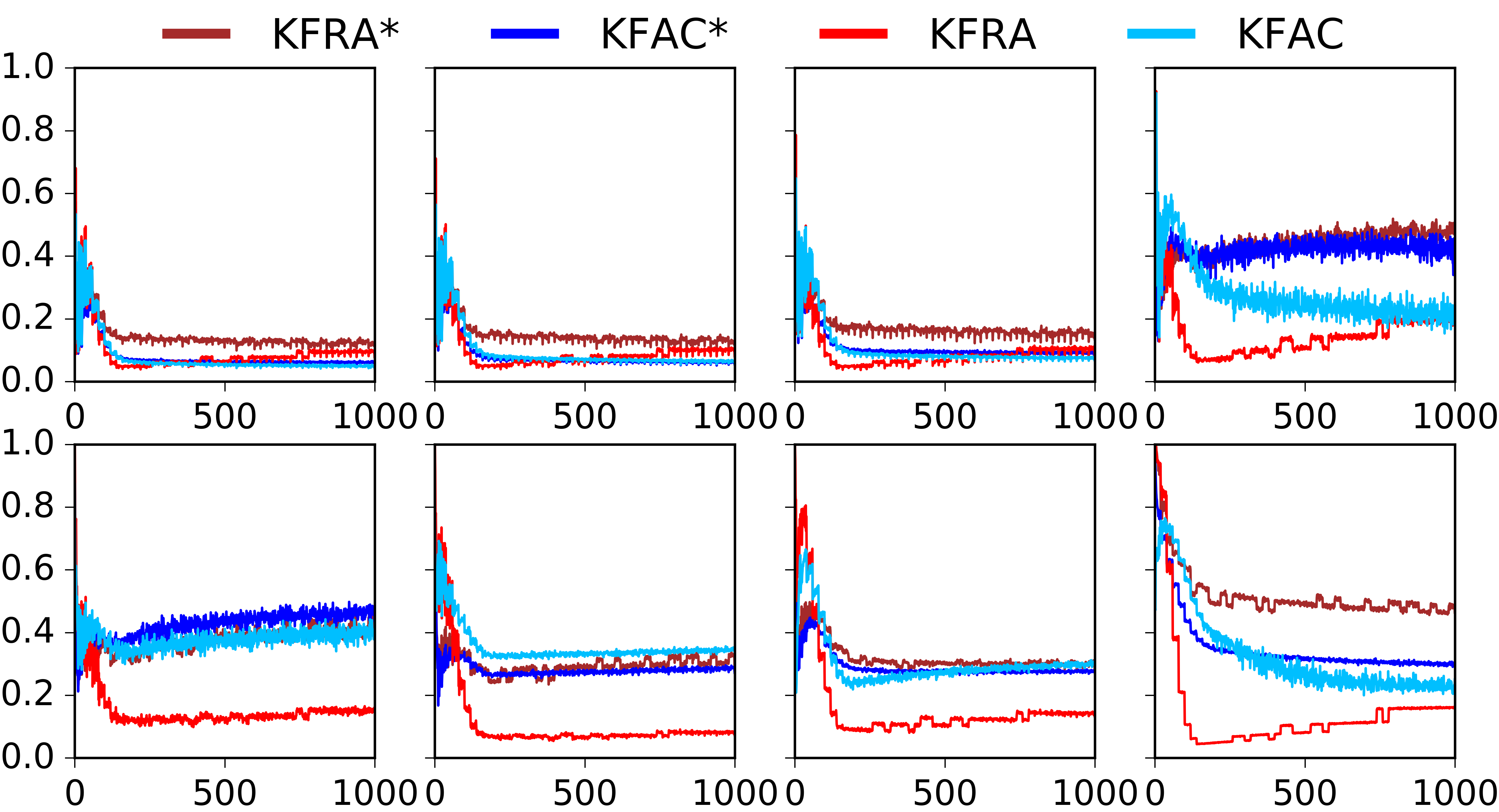}
        \caption{Block-diagonal Gauss-Newton}
        \label{fig:curves:cosbgn}
    \end{subfigure}%
    ~
    \begin{subfigure}[b]{0.5\textwidth}
    	\includegraphics[width=0.95\linewidth]{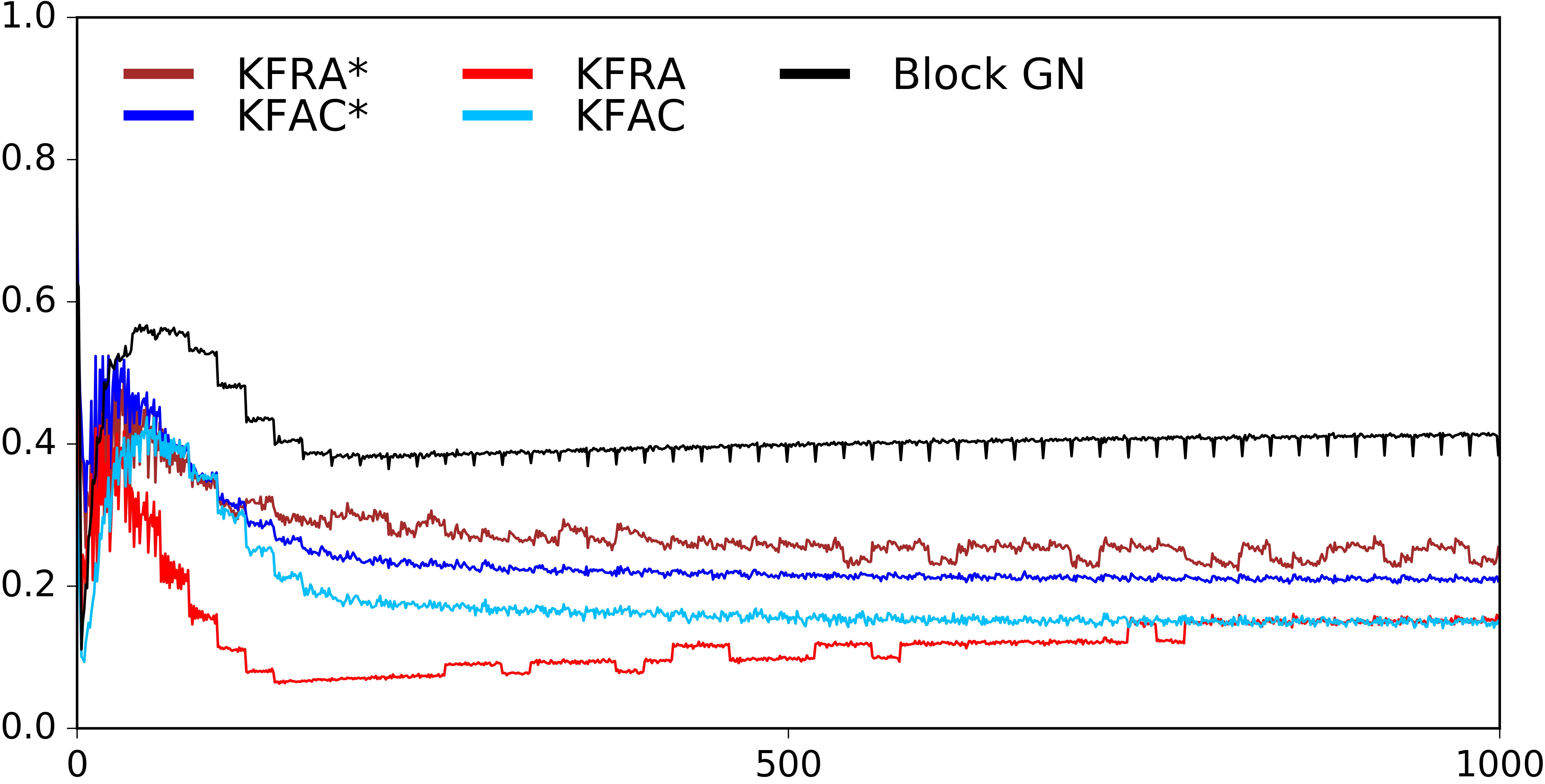}
    	\caption{Full Gauss-Newton}
    	\label{fig:curves:cosfgn}
    \end{subfigure}
    \caption{CURVES: Cosine similarity between the update vector per layer, given by the corresponding approximate method, $\widetilde{\delta}_\lambda$ with that for the block-diagonal GN (\protect\subref{fig:faces:cosbgn}) and the full vector with that from the full GN matrix (\protect\subref{fig:faces:cosfgn}).  The optimal value is 1.0. The $^*$ indicates approximate inversion in \eref{eq:approx:pi}. The $x$-axis is the number of iterations.  Layers one to four are in the top; five to eight in the bottom row. The trajectory of parameters we follow is the one generated by KFRA$^*$.}
    \label{fig:curves:updates}
\end{figure*}
}

\iftoggle{show_appendix_figs}{
\begin{figure*}[h!]
    \centering
    \begin{subfigure}[b]{0.5\textwidth}
        \includegraphics[width=0.95\linewidth]{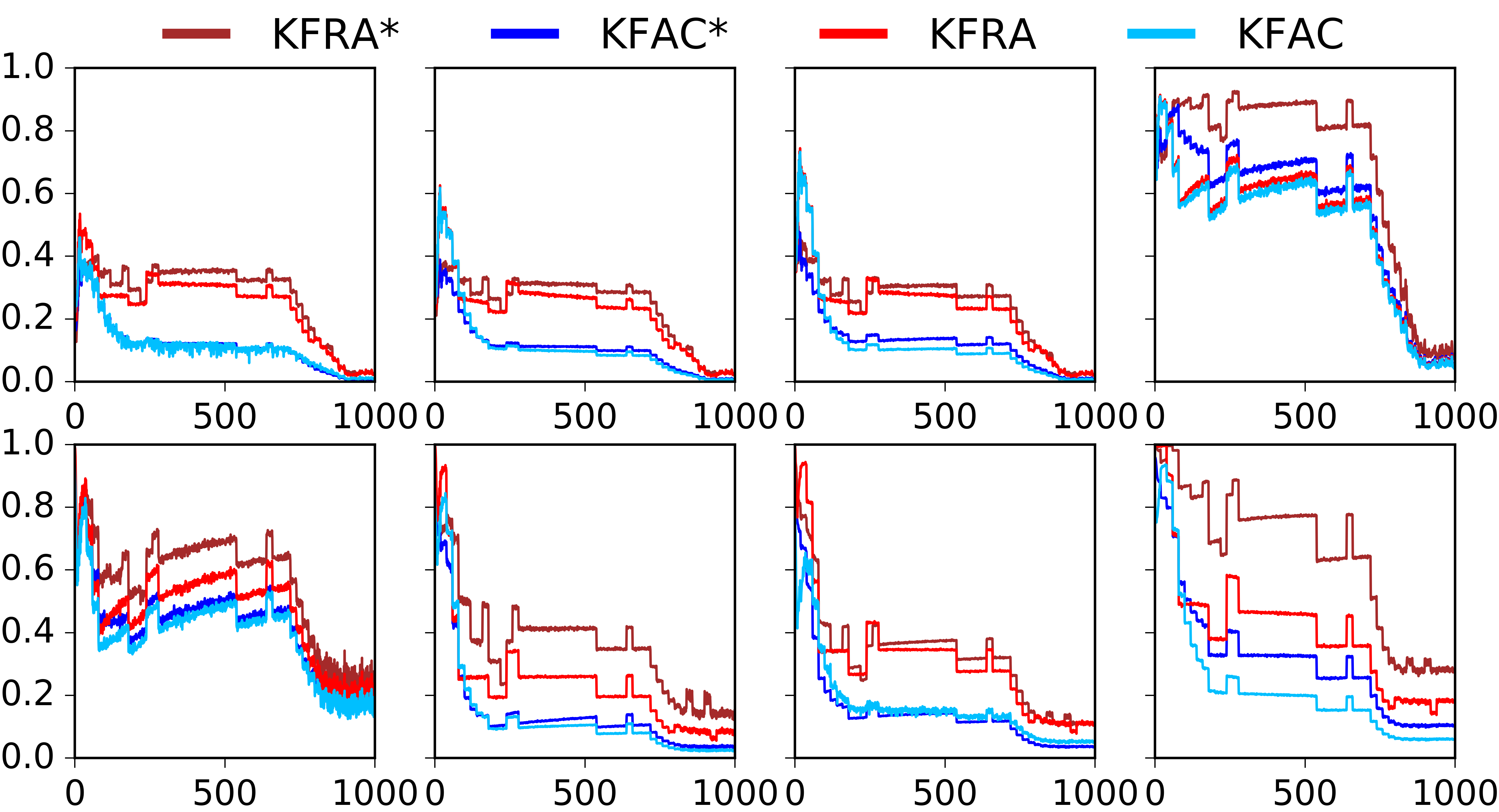}
        \caption{Block-diagonal Gauss-Newton}
        \label{fig:faces:cosbgn}
    \end{subfigure}%
    ~
    \begin{subfigure}[b]{0.5\textwidth}
    	\includegraphics[width=0.95\linewidth]{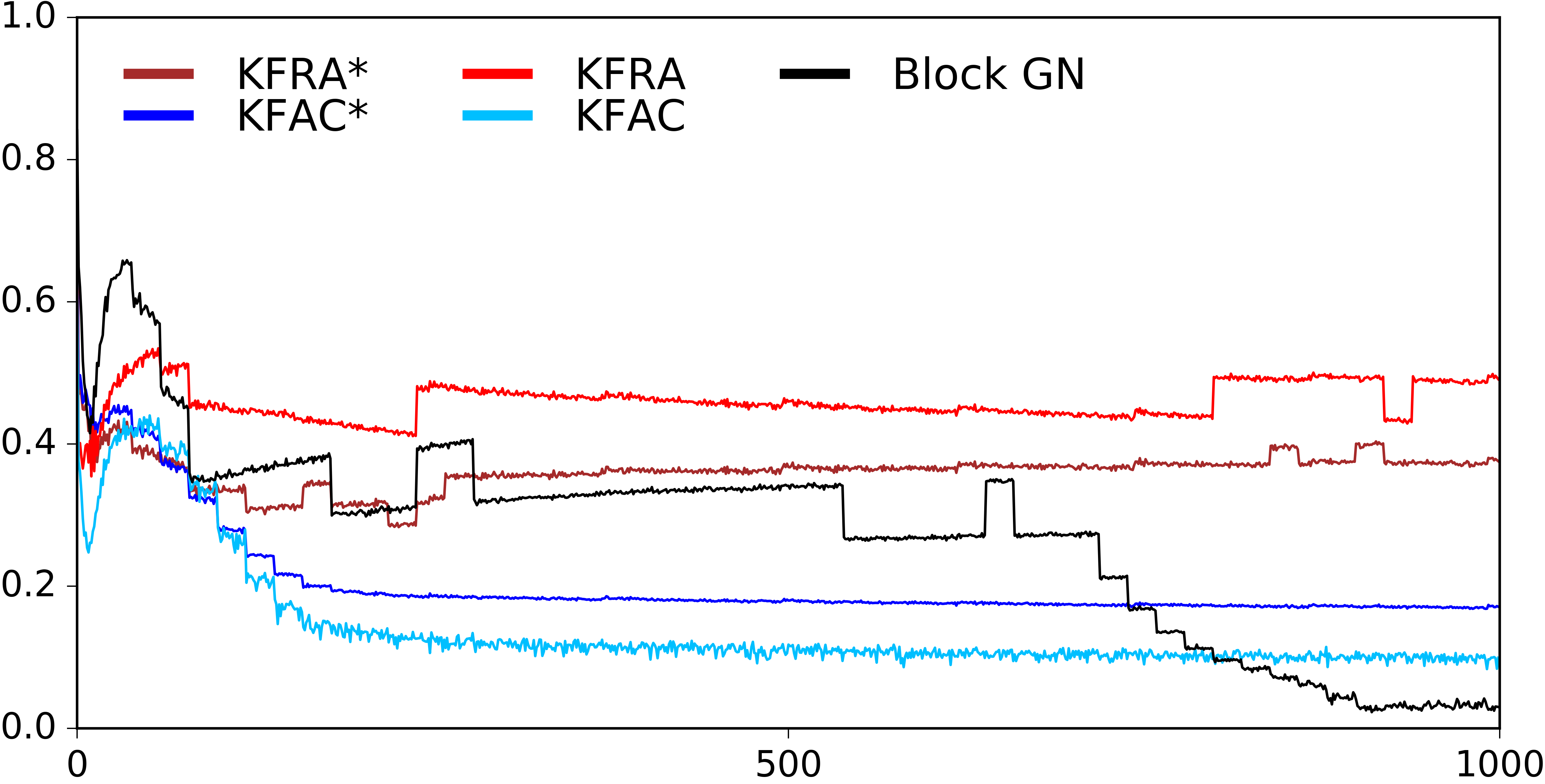}
    	\caption{Full Gauss-Newton}
    	\label{fig:faces:cosfgn}
    \end{subfigure}
    \caption{FACES: Cosine similarity between the update vector per layer, given by the corresponding approximate method, $\widetilde{\delta}_\lambda$ with that for the block-diagonal GN (\protect\subref{fig:faces:cosbgn}) and the full vector with that from the full GN matrix (\protect\subref{fig:faces:cosfgn}).  The optimal value is 1.0. The $^*$ indicates approximate inversion in \eref{eq:approx:pi}. The $x$-axis is the number of iterations.  Layers one to four are in the top; five to eight in the bottom row. The trajectory of parameters we follow is the one generated by KFRA$^*$.}
    \label{fig:faces:updates}
\end{figure*}
}

\iftoggle{show_appendix_figs}{
\begin{figure*}[h!]
    \centering
    \begin{subfigure}[b]{0.5\textwidth}
        \includegraphics[width=0.95\linewidth]{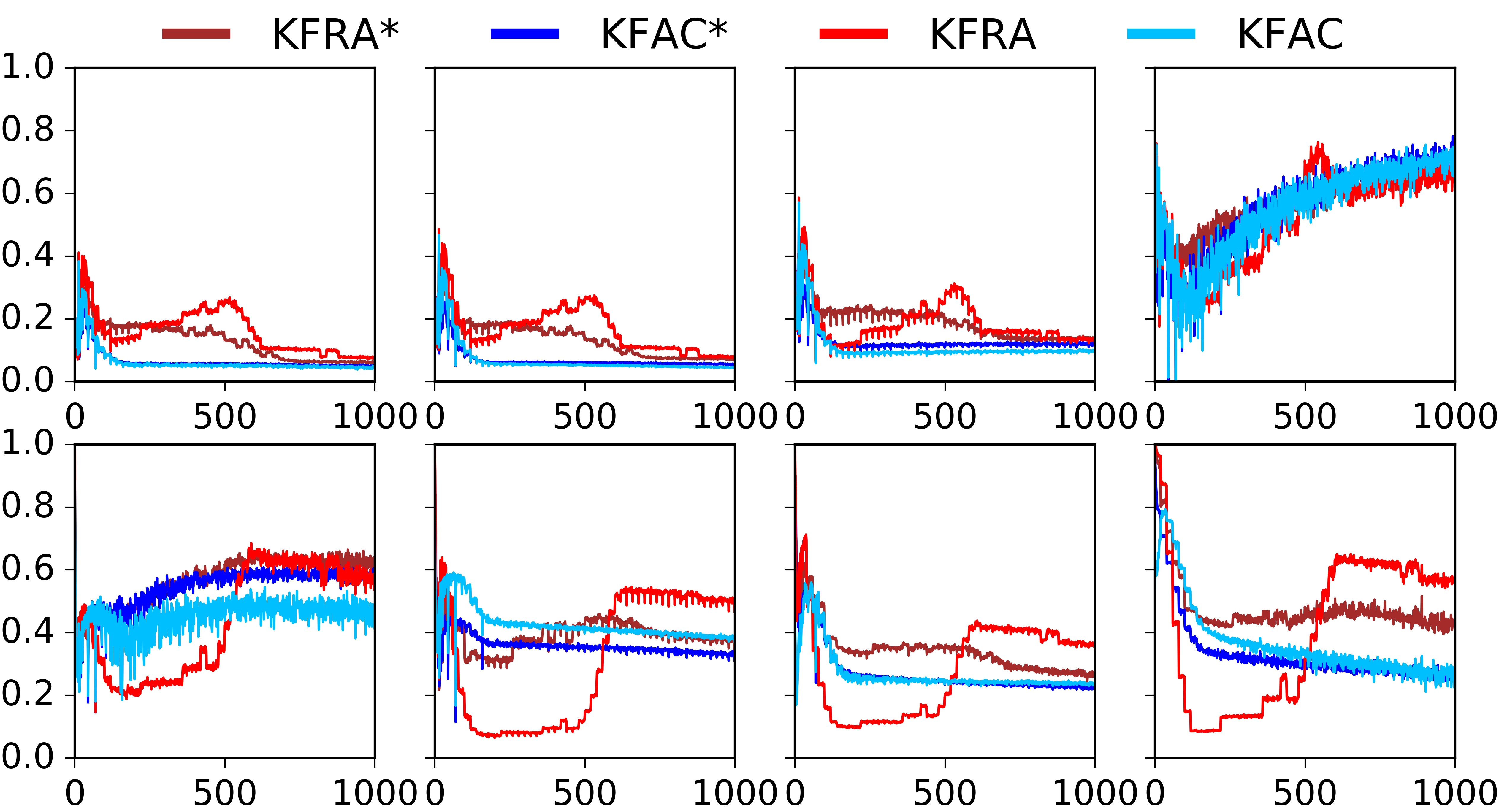}
        \caption{Block-diagonal Gauss-Newton}
        \label{fig:mnist:cosbgn}
    \end{subfigure}%
    ~
    \begin{subfigure}[b]{0.5\textwidth}
    	\includegraphics[width=0.95\linewidth]{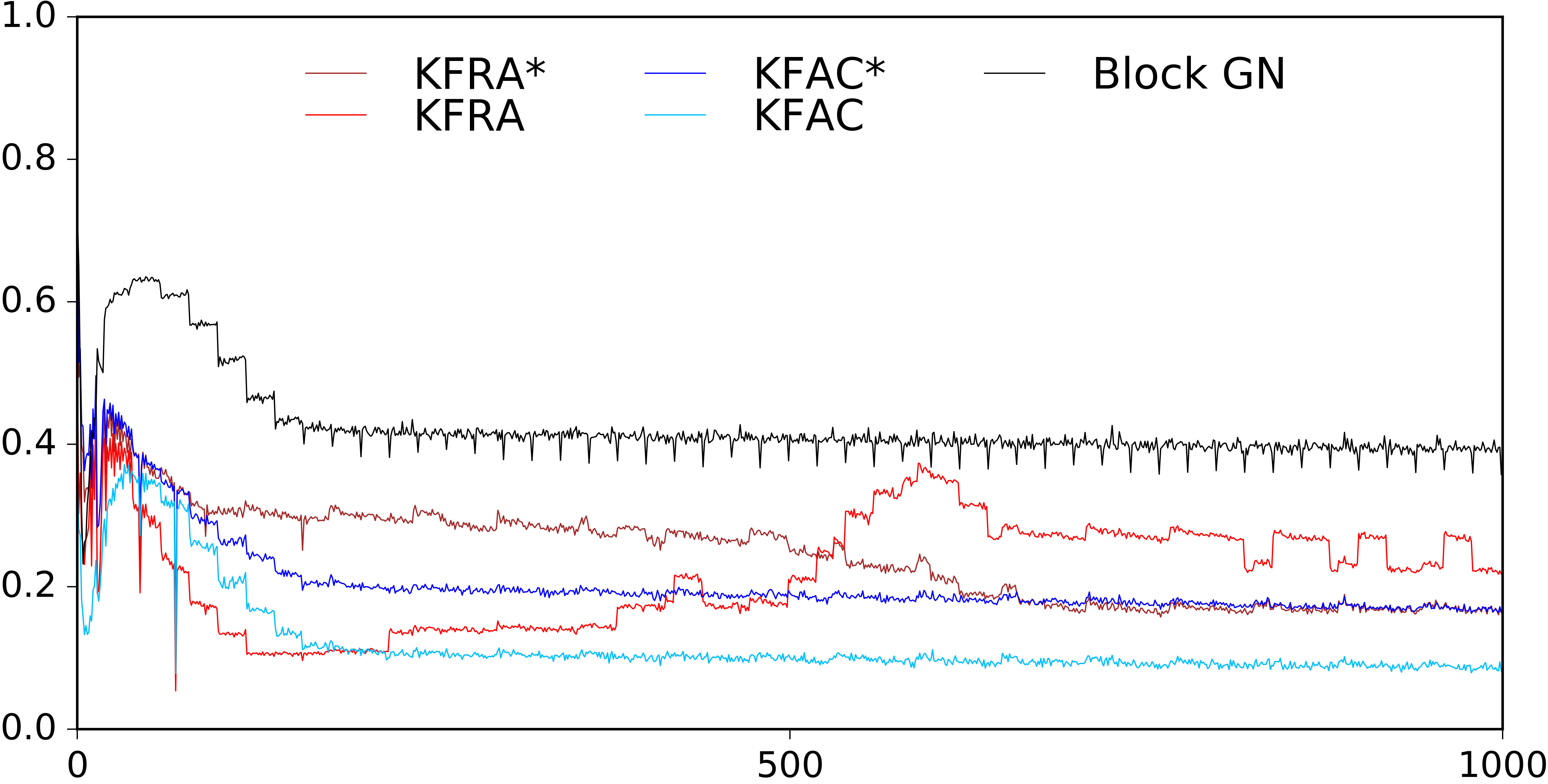}
    	\caption{Full Gauss-Newton}
    	\label{fig:mnist:cosfgn}
    \end{subfigure}
    \caption{MNIST: Cosine similarity between the update vector per layer, given by the corresponding approximate method, $\widetilde{\delta}_\lambda$ with that for the block-diagonal GN (\protect\subref{fig:faces:cosbgn}) and the full vector with that from the full GN matrix (\protect\subref{fig:faces:cosfgn}).  The optimal value is 1.0. The $^*$ indicates approximate inversion in \eref{eq:approx:pi}. The $x$-axis is the number of iterations.  Layers one to four are in the top; five to eight in the bottom row. The trajectory of parameters we follow is the one generated by KFRA$^*$.}
    \label{fig:mnist:updates}
\end{figure*}
}

To gain insight into the quality of the approximations that are made in the second-order methods under consideration, we compare how well the KFAC and KFRA parameter updates $\widetilde{\delta}$ are aligned with updates obtained from using the regularised block diagonal GN and the full GN matrix.
Additionally we check how using the approximate inversion of the Kronecker factored curvature matrices discussed in \aref{app:imp} impacts the alignment.

In order to find the updates for the full GN method we use conjugate gradients and the R-operator and solve the linear system $\bar{G}\delta=\nabla_\params f$ as in \citep{hessian-free-deep}. For the block diagonal GN method we use the same strategy, however the method is applied independently for each separate layer of the network, see \aref{app:imp}.

We compared the different approaches for batch sizes of 250, 500 and 1000. However, the results did not differ significantly.
We therefore show results only for a batch size of 1000.
In Figures \ref{fig:curves:updates} to \ref{fig:mnist:updates}, subfigure \ref{fig:faces:cosbgn} plots the cosine similarity between the update vector $\widetilde{\delta}_\lambda$ for a specific layer, given by the corresponding approximate method, and the update vector when using the block diagonal GN matrix on CURVES, FACES and MNIST.
Throughout the optimisation, compared to KFAC, the KFRA update has better alignment with the exact GN update.
Subfigure \ref{fig:faces:cosfgn} shows the same similarity for the whole update vector $\widetilde{\delta}$, however in comparison with the update vector given by the full GN update.
Additionally, we also show the similarity between the update vector of the block diagonal GN and the full GN approach in those plots.
There is a decay in the alignment between the block-diagonal and full GN updates towards the end of training on FACES, however this is most likely just due to the conjugate gradients being badly conditioned and is not observed on the other two datasets.

After observing that KFRA generally outperforms KFAC, it is not surprising to see that its updates are beter aligned with both the block diagonal and the full GN matrix.

Considering that (for exponential family models) both methods differ only in how they approximate the expected GN matrix, gaining a better understanding of the impact of the separate approximations on the optimisation performance could lead to an improved algorithm.

\clearpage

\section{Algorithm for a Single Backward Pass}

\begin{algorithm}
\label{alg:backprop}
\caption{Algorithm for KFRA parameter updates excluding heuristic updates for $\eta$ and $\gamma$}
\begin{algorithmic}
	\STATE {\bfseries Input:} minibatch $X$, weight matrices $W_{1:L}$, \transfers $f_{1:L}$, true outputs $Y$, parameters $\eta$ and $\gamma$
	
	\STATE {\itshape - Forward Pass -}
	\STATE $\act_0 = X$
	\FOR{$\lambda = 1$ {\bfseries to} L}
	\STATE $\pre_\lambda = W_\lambda \act_{\lambda - 1}$
	\STATE $\act_\lambda = \nonl_\lambda (\pre_\lambda)$
	\ENDFOR
	
	\STATE {\itshape - Derivative and Hessian of the objective -}
	\STATE $d_L = \dE{\pre_L}\Bigr|_{\pre_L}$
	\STATE $\wtl{\gn{L}} = \expect{\xhess{L}}\Bigr|_{\pre_L}$
	
	\STATE {\itshape - Backward pass -}
	\FOR{$\lambda = L$ {\bfseries to} $1$}
	\STATE {\itshape - Update for $W_\lambda$ -}
	\STATE $g_\lambda = \frac{1}{N} d_\lambda \act_{\lambda-1} \trans + \eta W_\lambda$
	\STATE $\wtl{Q} = \frac{1}{N} \act_{\lambda-1} \act_{\lambda-1}\trans $
	\STATE $\omega = \sqrt{\frac{Tr(\wtl{Q}) * dim(\wtl{\gn{}})}{Tr(\wtl{\gn{}}) * dim(\wtl{Q})}}$ 
	\STATE $k = \sqrt{\gamma + \eta}$
	\STATE $\widetilde{\delta}_\lambda = (\wtl{Q} + \omega k)^{-1} g_\lambda (\wtl{\gn{}} + \omega^{-1} k)^{-1}$
	\IF{$\lambda > 1$}
	\STATE {\itshape - Propagate gradient and approximate \prea Gauss-Newton -}
	\STATE $A_{\lambda-1} = \nonl'(\pre_{\lambda - 1})$
	\STATE $d_{\lambda-1} = W_\lambda \trans d_\lambda \odot A_{\lambda-1}$

	\STATE $\widetilde{\gn{}}_{\lambda - 1} = (W_\lambda \trans \wtl{\gn{}} W_\lambda) \odot \left( \frac{1}{N} A_{\lambda-1} A_{\lambda-1} \trans \right)$
	\ENDIF
	\ENDFOR
	
	\STATE $v = \J{\pre_L}{\theta} \widetilde{\delta}$ \quad \quad (using the R-op from \citep{rop})
	\STATE $\widetilde{\delta} \trans \bar{G} \widetilde{\delta} = v \trans \xhess{L} v$
	\STATE $\widetilde{\delta} \trans \bar{C} \widetilde{\delta} = \widetilde{\delta} \trans \bar{G} \widetilde{\delta} + (\tau + \eta)||\widetilde{\delta}||_2^2$
	
	\STATE $\alpha_* = - \frac{\widetilde{\delta}\trans \nabla f }{\widetilde{\delta}\trans \bar{C} \widetilde{\delta}}$
	\STATE $\delta_* =  \alpha_* \widetilde{\delta} $
	
	\FOR{$\lambda = 1$ {\bfseries to} $L$}
	\STATE $W_\lambda = W_\lambda + \delta_{*\lambda}$
	\ENDFOR
\end{algorithmic}
\end{algorithm}

\end{document}